\documentclass{article}
\usepackage{ijcai22}
\usepackage{subcaption}
\usepackage{numprint}
\usepackage{xspace}
\usepackage{adjustbox}

\usepackage{xspace}
\usepackage{times}
\usepackage{soul}
\usepackage{url}
\usepackage[hidelinks]{hyperref}
\usepackage[utf8]{inputenc}
\usepackage{graphicx}
\usepackage{amsmath}
\usepackage{amsthm}
\usepackage{booktabs}
\usepackage{algorithm}
\usepackage{algorithmic}
\usepackage{subcaption}
\usepackage{color}
\usepackage{balance}
\newtheorem{theorem}{Theorem}

\newcommand*{\acse}{\text{ACSEmployment}\xspace}
\newcommand*{\acsh}{\text{ACSPublicCoverage}\xspace}
\newcommand*{\acst}{\text{ACSMobility}\xspace}

\title{Spending Privacy Budget Wisely and Fairly}

\author{
Lucas Rosenblatt$^1$\footnote{Contact Author}\and
Joshua Allen$^2$\And
Julia Stoyanovich$^{1}$
\affiliations
$^1$New York University, 
$^2$Microsoft
}

\begin{document}

\maketitle

\begin{abstract}
Differentially private (DP) synthetic data generation is a practical method for improving access to data as a means to encourage productive partnerships. One issue inherent to DP is that the ``privacy budget'' is generally ``spent'' evenly across features in the data set. This leads to good statistical parity with the real data, but can undervalue the conditional probabilities and marginals that are critical for  predictive quality of synthetic data.  Further, loss of predictive quality may be non-uniform across the data set, with subsets that correspond to minority groups potentially suffering a higher loss.  

In this paper, we develop ensemble methods that distribute the privacy budget ``wisely'' to maximize predictive accuracy of models trained on DP data, and ``fairly'' to bound potential disparities in accuracy across groups and reduce inequality.  Our methods are based on the insights that feature importance can inform how privacy budget is allocated, and, further, that per-group feature importance and fairness-related performance objectives can be incorporated in the allocation.  These insights make our methods tunable to social contexts, allowing data owners to produce balanced synthetic data for predictive analysis.



\end{abstract}

\section{Introduction}

Methodologies from data science and machine learning have become pervasive in recent years, leading to tremendous progress 
across disciplines. Common to most of these applications is a reliance on data for model creation and evaluation. Often, especially when the data in question stems from a social context, models in deployment may risk adversely affecting an individual member of a sample population. These adverse effects can take on many forms: for example, by treating individuals unfairly or by violating their privacy. 

Generally agreed upon ethics and country-specific laws govern what constitutes a ``fair'' algorithmic decision making process. Some ML models have been shown to directly discriminate against an individual in a population based on their race or gender~\cite{propublica,DBLP:journals/bigdata/Chouldechova17,obermeyer}, breaking ethical and regulatory norms. Motivated by these findings, a great deal of research on fairness in ML has worked to measure and address these offenses, and to quantify the inherent tradeoffs. 

Ethics and laws also protect the rights of individual data providers by ensuring a high level of privatization. 
The differentially private promise, if used correctly, ensures that any inferences conducted on data do not reveal whether a single individual's information (including, for example, their gender or race) was included in the data for analysis \cite{dwork2006calibrating}. 
Differential privacy (DP) can both prevent leakage of an individual's data and allow their sensitive attributes to be used during model training, supporting a balanced algorithmic decision making process \cite{jagielski2019differentially}.  

{\bf Challenges.} A practical method of operationalizing DP is through the generation of differentially private synthetic data~\cite{dwork2009complexity,hardt2010simple,torkzadehmahani2019dp,rosenblatt2020differentially,vietri2020new,mckenna2021winning}. Here, the goal is to develop general purpose DP data synthetizers that perform well across a number of metrics, with high statistical and predictive utility. However, this has proven to be a difficult task. One  issue inherent to DP is that the ``privacy budget,'' or $\epsilon$ value, is generally ``spent'' (or distributed) evenly across features in the data set. This leads to good statistical parity with the real data, but can often undervalue the conditional probabilities and marginals that are critical for preserving the predictive quality of the data.  Further, utility loss may be non-uniform across subsets of a data set, with subsets that correspond to minority or historically under-represented groups potentially suffering a higher loss. 

\begin{table}[t!]
\centering
\small 
\caption{Predicting binary employment status on \acse (state=CA), where groups with existing imbalance (data contains $61.4$\% white and $4.8$\% black individuals) are impacted adversely by privatization in data release.  DP synthesizers MST and FSQ-Bal (our method), with $\epsilon=e^3$, average over 10 runs.}
\begin{tabular}{lccccc}
\toprule
 & & \multicolumn{3}{c}{Accuracy / FNR (\%)} \\  
Group  & Base rate (\%) & Real & MST & FSQ-Bal \\
\midrule
 Overall  & 45.6 & 74.3 / 12 &  65.2 / 10 & 72.0 / 9  \\
 White    & 45.8 & \textbf{74.5 / 12}  & \textbf{65.6 / 10}  & \textbf{72.4 / 9}  \\
 Black    & 39.3 & \textbf{71.8 / 16}  & \textbf{58.6 / 21}  & \textbf{66.8 / 9}    \\
\bottomrule
\end{tabular}
\label{tab:example}
\vspace{-0.1in}
\end{table}

As an example, consider the task of predicting (binary) employment status on the \acse data set~\cite{ding2021retiring}.  Table~\ref{tab:example} shows prediction accuracy of a logistic regression classifier on real data and on DP synthetic data generated by MST~\cite{mckenna2021winning}, a state-of-the-art synthesizer.  Observe that accuracy on the real data is 74.3\% overall, and that it's 2.7\% higher for the white group, which constitutes the majority of the data set, compared to the black group.  Accuracy of the model trained on MST-generated data is substantially lower at 65.1\% overall,  with an exacerbated 7\% disparity in the accuracy between the two racial groups.

{\bf Key ideas.} In this paper, we develop ensemble synthetic data methods that distribute the privacy budget ``wisely'' to maximize predictive accuracy, and ``fairly'' to bound potential disparities in accuracy across sub-populations.  Our methods are based on three key insights.  The first is that feature importance can inform how the privacy budget is allocated to individual features. The second is that fairness-related performance objectives 
can be incorporated into that budget allocation. The third is that, given a partial segmentation of data synthesis into tunable models for standalone features, 
one can optimize for a specific accuracy metric, while encouraging desirable predictive properties of synthetic data for sub-populations. 

For example, in the employment scenario, a false negative is considered particularly harmful to an individual. We thus may want to minimize FNR overall while equalizing FNR between groups, all while maximizing predictive fidelity of synthetic data. Table~\ref{tab:example} shows accuracy and FNR for one of the methods we propose, FSQ-Bal, where we achieve high overall accuracy of 72\%, approaching accuracy on real data, while lowering FNR overall and for both groups to 9\%, outperforming real data on this metric. 

{\bf Social impact and relevance.} Procedures for generating DP synthetic data that explicitly consider protected classes and fairness 
must be developed.
Otherwise, as is well documented~\cite{ganev2021robin}, and as we showed in our example in Table~\ref{tab:example}, DP synthesizers will continue to  have an adverse effect on underprivileged groups. The approach we take in this work leads to DP synthesizers that are interpretable and tunable to social contexts, allowing data owners to produce balanced synthetic data for predictive analysis. 

\section{Overview of Our Approach}
\begin{figure}[!htb]
\centering
\includegraphics[width=7cm]{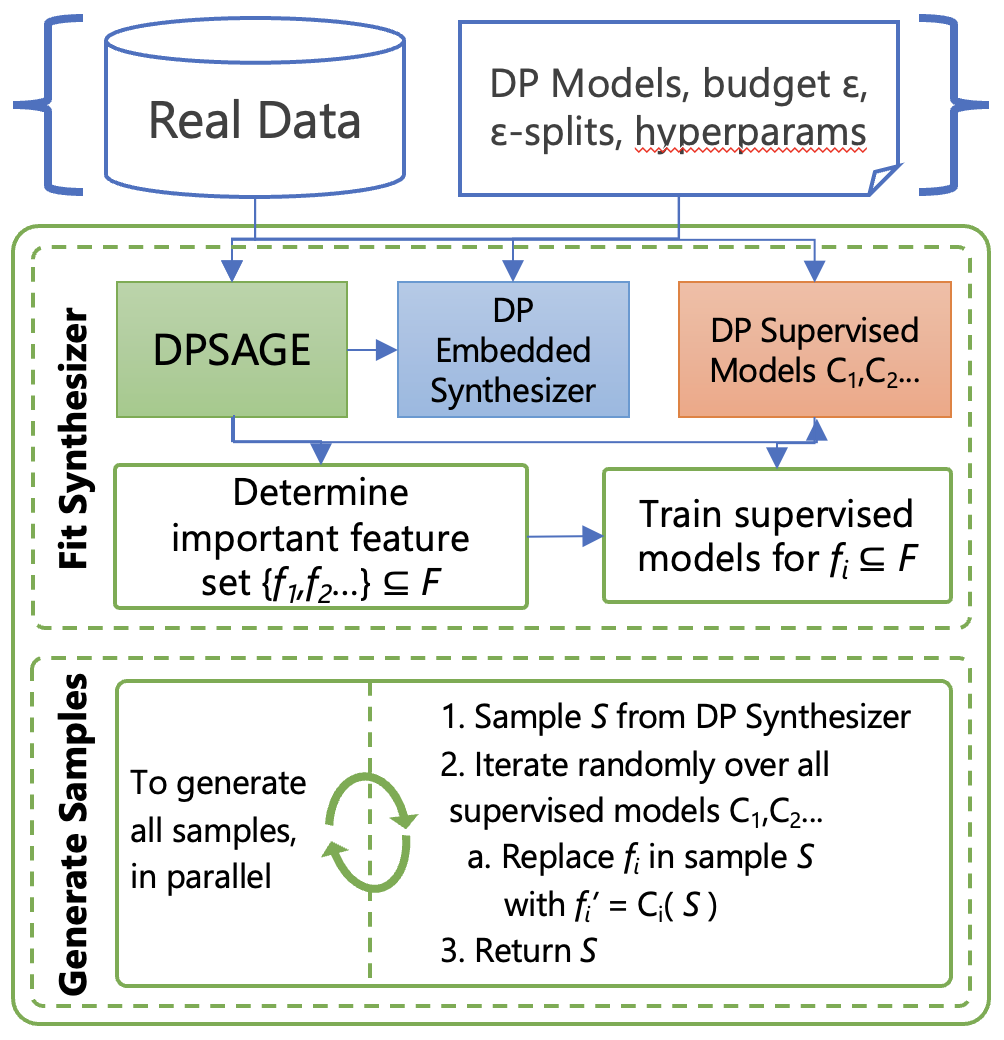}
\caption{The architecture of SuperQUAIL, outlining major components and summarizing the iterative approach to sampling.}
\label{fig:sqdiagram}
\vspace{-0.1in}
\end{figure}

{\bf Challenges in DP synthetic data generation.} It is well known that generating DP synthetic data, though promising for many tasks, suffers from an array of computational challenges. One can generate a synthetic dataset with stringent privacy guarantees, and that dataset can preserve complex marginal distributions across many attributes. However, a variant of the ``curse of dimensionality'' inherently limits practical implementations of these theoretically strong synthesizers  \cite{hardt2010simple,ullman2011pcps,bowen2020comparative}. Much of the recent work on DP data synthesis has involved balancing trade-offs between fidelity to real data and practical predictive performance \cite{vietri2020new,mckenna2021winning,tao2021benchmarking}. 

Past work on DP synthetic data focused on privatized Bayesian networks and conditional GANs \cite{zhang2017privbayes,torkzadehmahani2019dp,jordon2018pate}, while recent work has proposed new approaches to ``querying'' from real data with improved methods for incorporating queries into a mock-distribution from which one can sample \cite{arnold2020really,bowen2020comparative}. Our approach falls into the latter category: we offer an ensemble method that treats the results of DP classification algorithms as ``queries'' to improve the fidelity of independently drawn samples, augmenting an existing DP data synthesis method. 

However, most existing algorithms offer unsatisfactory explorations of the social contexts where they are most often employed. Given that DP serves to rigorously preserve privacy rights for individuals, we can safely assume that the majority of use will be in high-stakes scenarios, such as the US Census \cite{abowd2018us}. The societal impact and potential harms of DP methods, then, need attention. Important existing work details ``fair'' DP approaches to classification \cite{jagielski2019differentially,ding2020differentially}, specifically a compelling post-processing method that utilizes privacy budget to ``repair'' inequities introduced by DP noise \cite{pujol2020fair}. We believe that similar work must be extended to address adverse impact and potential predictive harms to minority groups in DP synthetic data, an unfortunate byproduct of existing synthesizers, explored in depth by \cite{ganev2021robin}. We propose our methods, SuperQUAIL and its fair derivatives, as tools to ameliorate synthetic data harms without sacrificing much data utility.

{\bf Overview of SuperQUAIL.} Our work leverages existing methods of generating DP data and feature importance in order to generate high quality and balanced tabular synthetic data for an {\it a priori} known supervised learning task. Figure~\ref{fig:sqdiagram} gives an overview of SuperQUAIL.  The method operates in two steps: {\bf (1)} fit the synthesizer , ``wisely'' allocating the privacy budget to different features to optimize for predictive utility of the DP sample and {\bf (2)} generate samples to ensure that a ``fair'' (in the sense of performance across groups) classifier can be trained on the DP sample.

To support step {\bf (1)}, we present DPSAGE, an intuitive DP modification to SAGE~\cite{DBLP:conf/nips/CovertLL20}, a popular explainability framework, and show that it reasonably mimics importance values drawn from real (non-privatized) data at various $\epsilon$ levels. With a feature ranking from DPSAGE, we can then improve the allocation of our privacy budget \textit{within} our synthesizer, creating data tuned to a specific task. We do this by iteratively ``querying'' DP classifiers, each targeting one of the important predictive features from DPSAGE, to improve samples generated by an embedded DP synthesizer. We describe this process in detail in Section~\ref{sec:wise}.

To support step {\bf (2)}, through an intuitive modification to the SuperQUAIL method, we demonstrate how to tune the \text{fairness} of the resulting synthetic data with respect to protected classes (e.g., race), thus allowing for supervised learning models with reduced group disparities to be trained on the synthetic data. 
We describe this process in detail in Section~\ref{sec:fair}.

\vspace{-0.2cm}
\section{SuperQUAIL}
\label{sec:wise}
SuperQUAIL embeds a DP feature importance method, a DP synthesizer (we use MST), and DP classifiers into an ensemble model. It uses the feature importance values to assign the privacy budget across these feature-respective classifiers, as well as the target feature, thus prioritizing the preservation of a specific ``marginal'' (namely, between the target variable and the most important features for target prediction). We show SuperQUAIL in Algorithm~\ref{alg:sq} and discuss it below.
\subsection{Feature Importance} 
\label{sec:methods:sq}
SAGE (Shapley Additive Global Importance) is a SHAP value based global feature importance measure. It is given by the following equation, where $v_{f^*}$ is the predictive power that some model $f$ derives from the subset of features $S \in D$, $d$ is the number of features, and $I(Y;X_i|X_S)$ denotes mutual information: predictive power of feature $X_i$ given all other features $X_S$ when predicting $Y$: 
\begin{small}
$$\phi(v_{f^*}) = \frac{1}{d} \sum_{S\in D \sim {i}} {d-1 \choose |S|}^{-1} I(Y;X_i|X_S)$$
\end{small}

Intuitively, this calculation means taking the weighted average of conditional mutual information (or how good some feature $X_i$ is at reducing uncertainty about prediction $Y$). Aggregating SHAP values in this way produces a global feature importance for each feature, but practically SAGE values are approximated by sampling real data to determine feature importance of a model trained on that data \cite{DBLP:conf/nips/CovertLL20,lundberg2017unified}. We chose this explainability model based on its state-of-the-art performance and its reliance on sampling (which made it straightforward to create a DP version). Some work exists on DP explainability \cite{harder2020interpretable,nori2021accuracy}, but these methods are tied directly to predictive modeling techniques. We sought a standalone feature importance method for integration into our synthesizer.

\begin{algorithm}[tb]
\small 
\caption{DP SAGE}
\label{alg:dpsage}
\textbf{Input}: Real data $D$, Untrained DP Sampler $S$ and DP supervised learning model $C$, SAGE models $PermSAGE$ and $ImpSAGE$ \\
\textbf{Parameter}: Budget $\epsilon$, split $\gamma$, model hyperparameters $H_S$ and $H_C$\\
\textbf{Output}: Feature importance $F$, trained $S$ and trained $C$ (so as not to waste budget)
\begin{algorithmic}[1] 
\STATE Perform $\epsilon$ privacy split, $\epsilon_S = \gamma * \epsilon$ and $\epsilon_C = 1 - (\gamma * \epsilon)$
\STATE Train $C(\epsilon_C, D, H_S)$. 
\STATE Initialize DP sampling method $S(\epsilon_S, D, H_C)$.
\STATE Initialize SAGE imputation method $ImpSage(S)$ 
\STATE Permutation estimator $PermSAGE(C, ImpSage(S))$
\STATE \textbf{return} Feature ranking $F$ and values $V_F$, also $S$, $C$
\end{algorithmic}

\end{algorithm}

For our DP feature importance method DPSAGE, shown in Algorithm~\ref{alg:dpsage}, we audit a DP supervised learning model. When estimating the SAGE values for that model, we use the standard permutation estimator and marginal imputer, but when sampling, we provide private samples, thus ensuring DP by composition (utilizing the combined privacy budget of both model and sampling technique). For our tests we use a DP logistic regression model~\cite{chaudhuri2011differentially} and MST for sampling \cite{mckenna2021winning}. We experimentally evaluate this method in Section~\ref{sec:exp}.

\subsection{Better Preserving Conditional Probabilities}
\cite{tao2021benchmarking} demonstrate that marginal-based methods for synthesis, which rely on some set of measurements of low-order marginals that fit a graphical model, have recently produced impressive results in both statistical and predictive utility. Specifically, the MST model, which uses HDMM \cite{mckenna2018optimizing}, relies on 2-way and 3-way marginals. However, given that the low-order marginals may struggle to capture complex feature relationships (i.e., higher-order marginals) due to a number of constraints including privacy budget, the predictive utility of the synthetic data may suffer when the task is difficult (i.e., complex conditioning on 4+ features). The main insight of this work, from a synthetic data perspective, is that one way to capture a high-order conditional probability is to use a DP predictive model. For this example, we will consider a logistic regression for ease of analysis, though random forests and other models produced similar results in our experiments.

The logistic regression model approximates a Bayes optimal decision boundary 
\begin{small}
$$\hat{y} = \underset{Y}{\arg\max} \text{ }Pr(Y=y | D)$$ 
\end{small}
with a linear decision boundary. Specifically, the logit function approximates the probability with: 
\begin{small}
$$Pr(Y=1 | D) = \frac{exp(\beta_0 + \sum_1^i \beta_i D_i)}{1+exp(\beta_0 + \sum_1^i \beta_i D_i)}$$ 
\end{small}
In a sense, this can be thought of as a marginal query, answering the following question: Given a sample of values for all but one dimension, what value does the target dimension take on with the highest probability? In essence, the DP version of logistic regression works by noising the objective function using the Laplace mechanism, proportional to the given $\epsilon$, thus approximating a full fidelity model \cite{chaudhuri2011differentially}.

What if, from a starting sample $S$, we iteratively generate (in random order) replacement values for a few specific dimensions of $S$ (including the target, or task specific, feature)? Might we be able to preserve a more complex conditional relationship than a 2-way or 3-way marginal? In an optimal setting, without noise, consider a hypothetical sample $S$, with sample features $X_S=i$, $Y_S=j$, $Z_S=k$ (loosely conditioned on each other), where we iterate over predictive models for $Y_S$, $X_S$. We could determine a value $y$ for replacement into the $Y_S$ feature of sample $S$ as follows: 
\begin{small}
\begin{align*}
y = &\underset{Y_S}{\arg\max}\text{ }  Pr(Y_S = y |  x = \nonumber\\ &~~~~~~~~\underset{X_S}{\arg\max}\text{ } Pr(X_S = x | Y_S=j, Z_S=k), Z_S=k)
\end{align*}
\end{small}

Note that this nested conditioning produces a value for a feature with scalably complex dependency. However, given our fixed $\epsilon$ privacy budget, and the expense of training individual predictors for each feature, we need to be careful about how we distribute the $\epsilon$ budget between features. This problem motivated our exploration into DPSAGE, which proved an effective method of selecting the top $k$ features for iterative conditioning without utilizing excess budget. Other methods of feature selection merit further study.

\begin{figure*}[t!]

\begin{subfigure}{\linewidth}
\centering
\includegraphics[width=.33\textwidth]{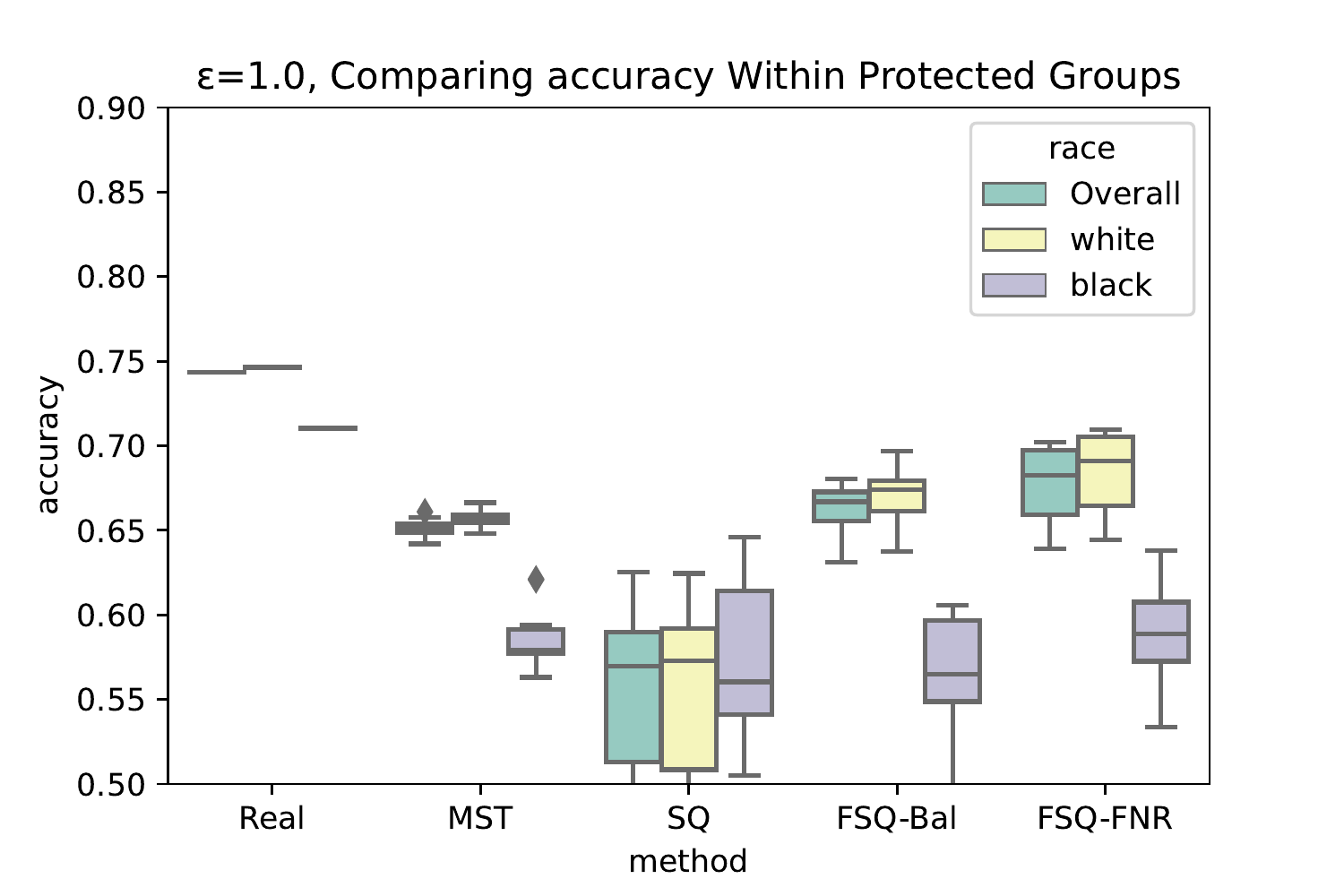}\hfill
\includegraphics[width=.33\textwidth]{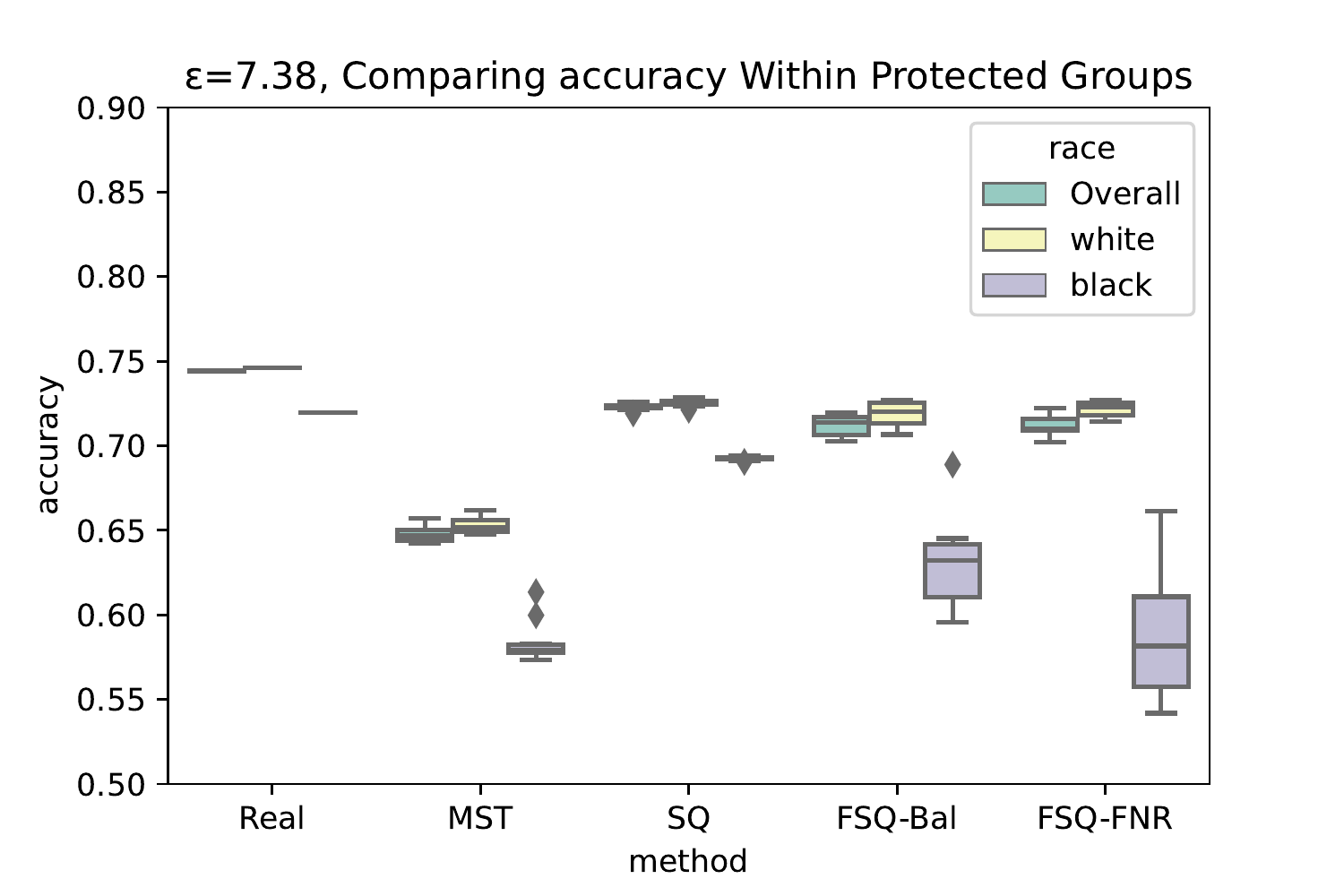}\hfill
\includegraphics[width=.33\textwidth]{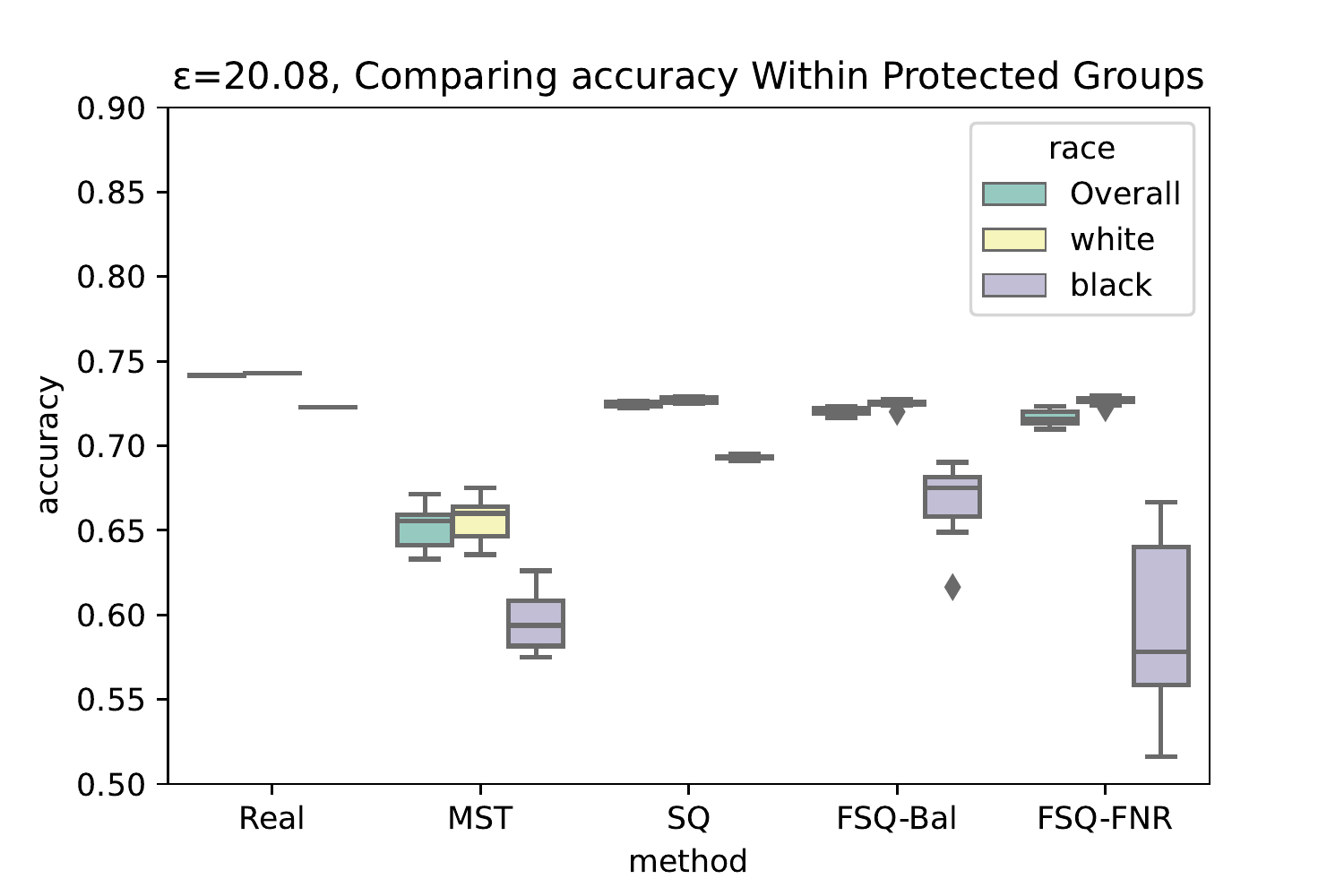}\hfill
\vspace{-0.1in}
\caption{Accuracy score at varying $\epsilon \in \{e^0,e^2,e^3\}$ on \acse.}
\label{fig:accuracy}
\vspace{-0.1in}
\end{subfigure}\par\medskip

\begin{subfigure}{\linewidth}
\centering
\includegraphics[width=.33\textwidth]{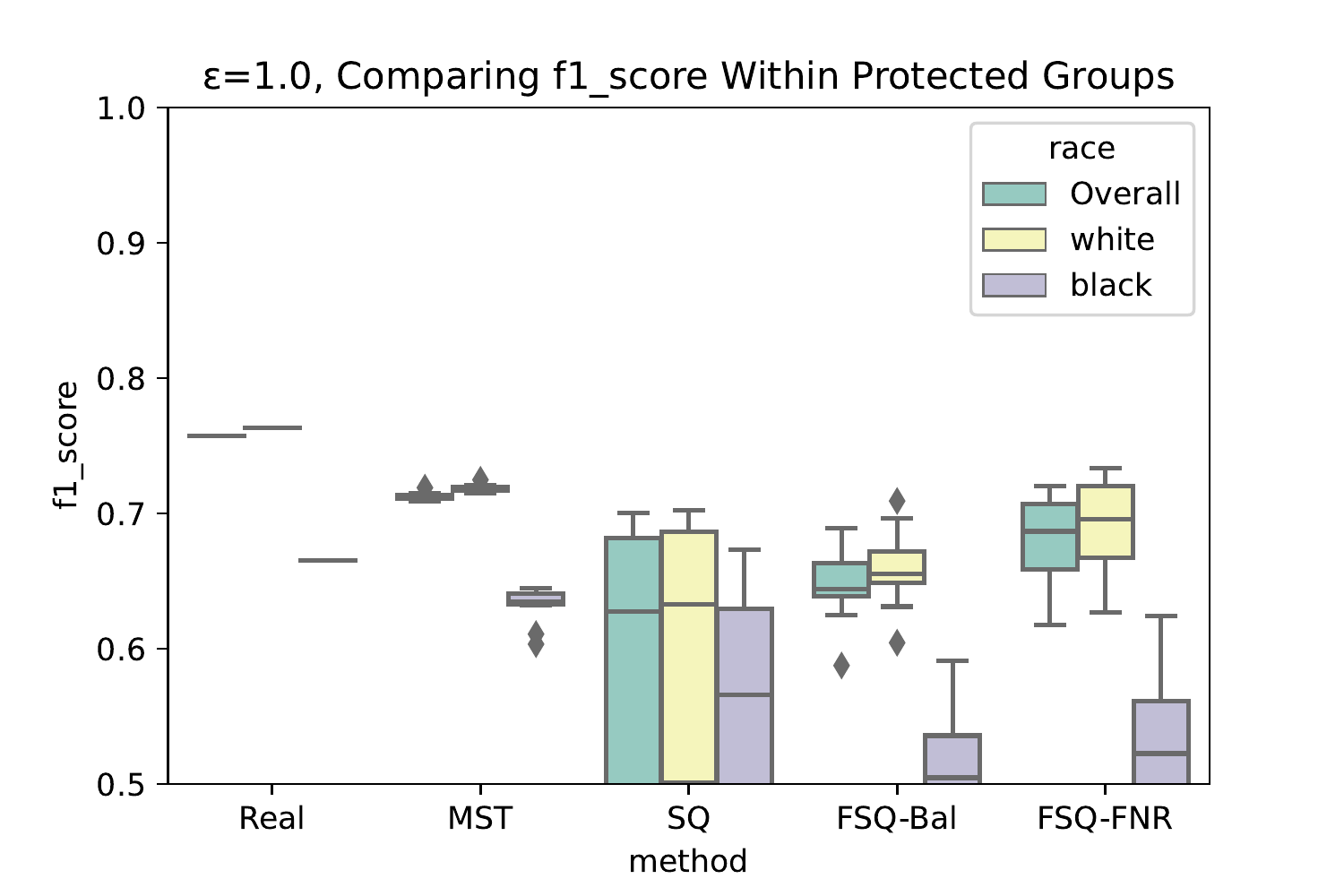}\hfill
\includegraphics[width=.33\textwidth]{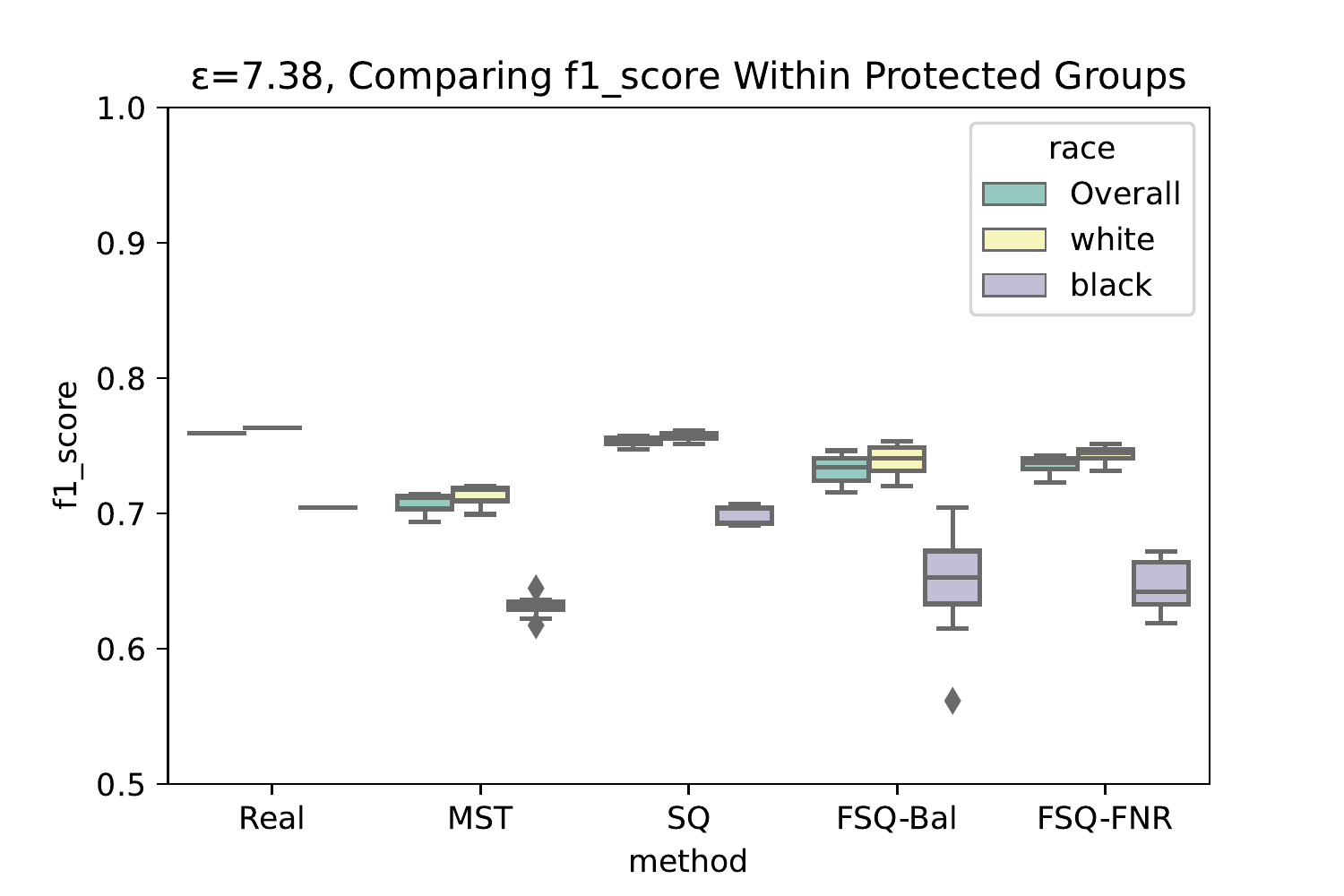}\hfill
\includegraphics[width=.33\textwidth]{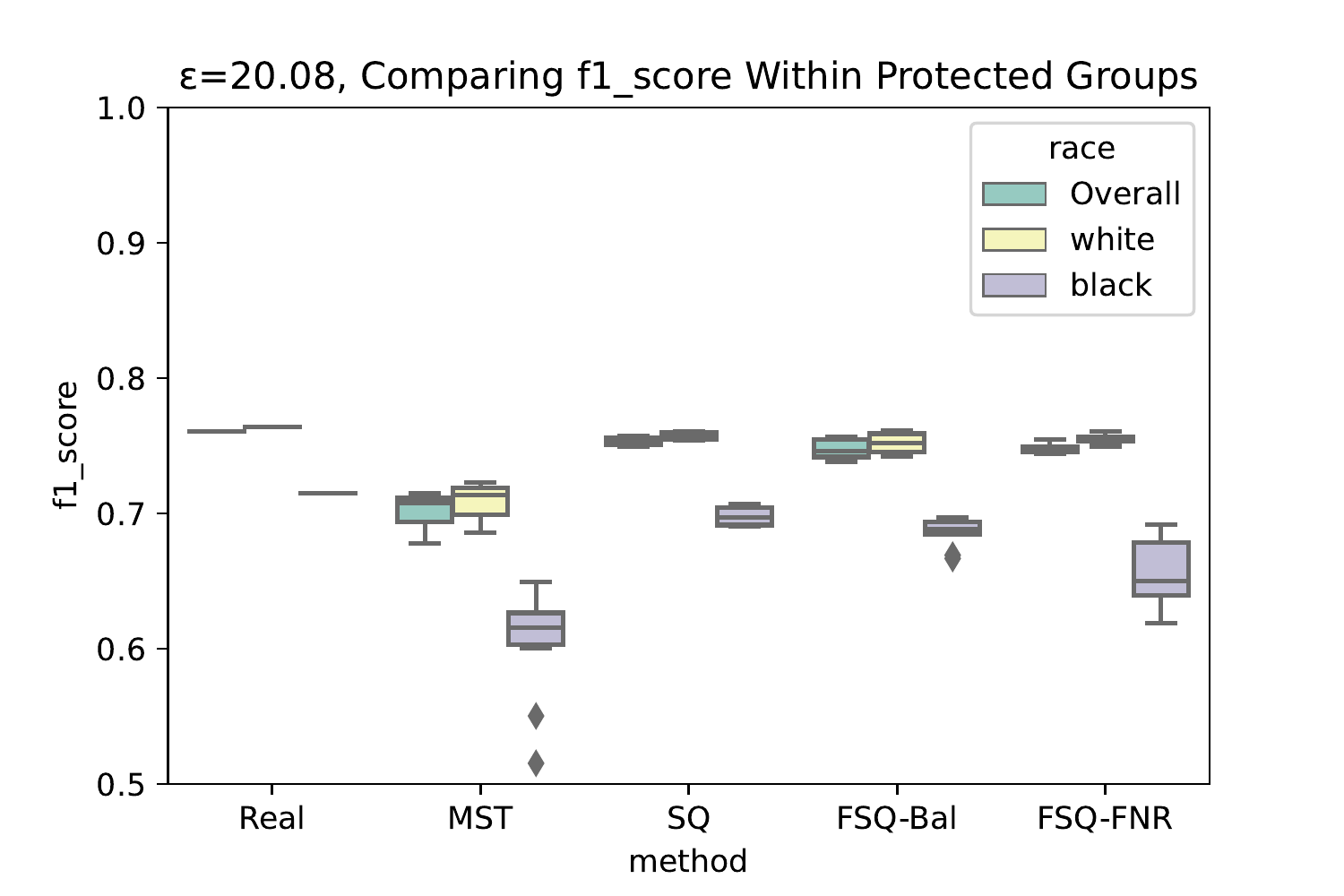}\hfill
\vspace{-0.1in}
\caption{F1 score at varying $\epsilon \in \{e^0,e^2,e^3\}$ on the \acse.}
\label{fig:f1}
\vspace{-0.1in}
\end{subfigure}\par\medskip

\begin{subfigure}{\linewidth}
\centering
\includegraphics[width=.33\textwidth]{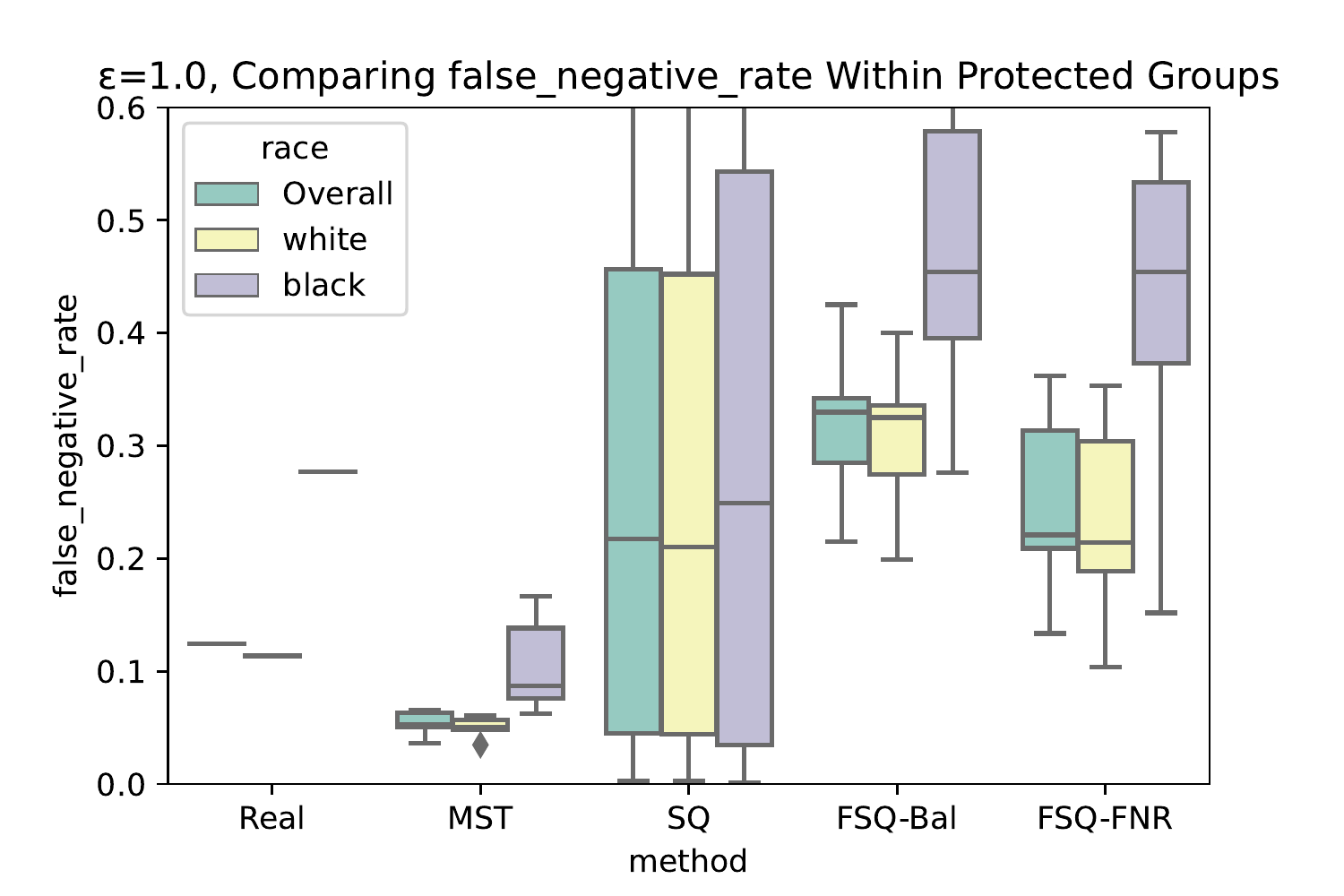}
\hfill
\includegraphics[width=.33\textwidth]{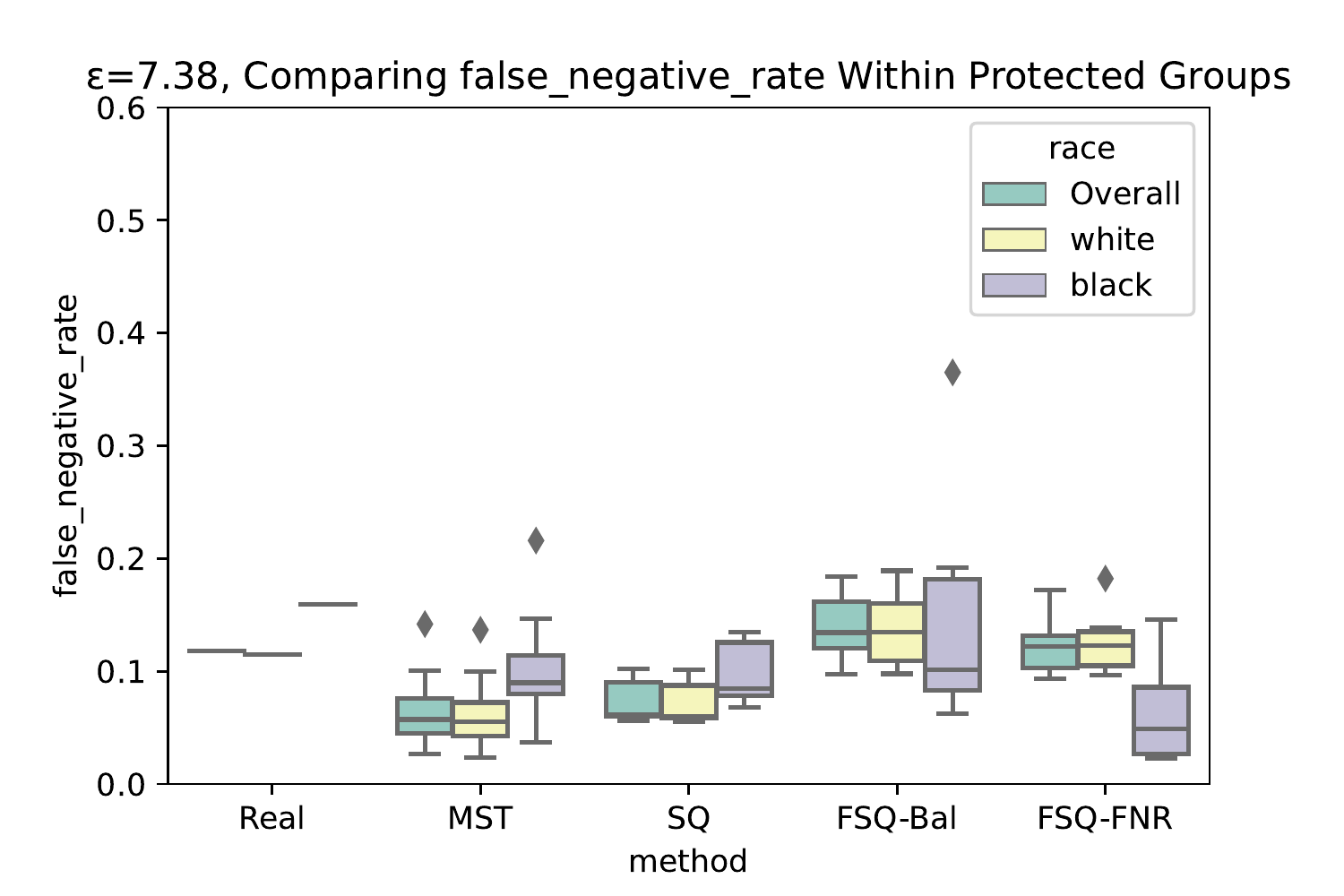}\hfill
\includegraphics[width=.33\textwidth]{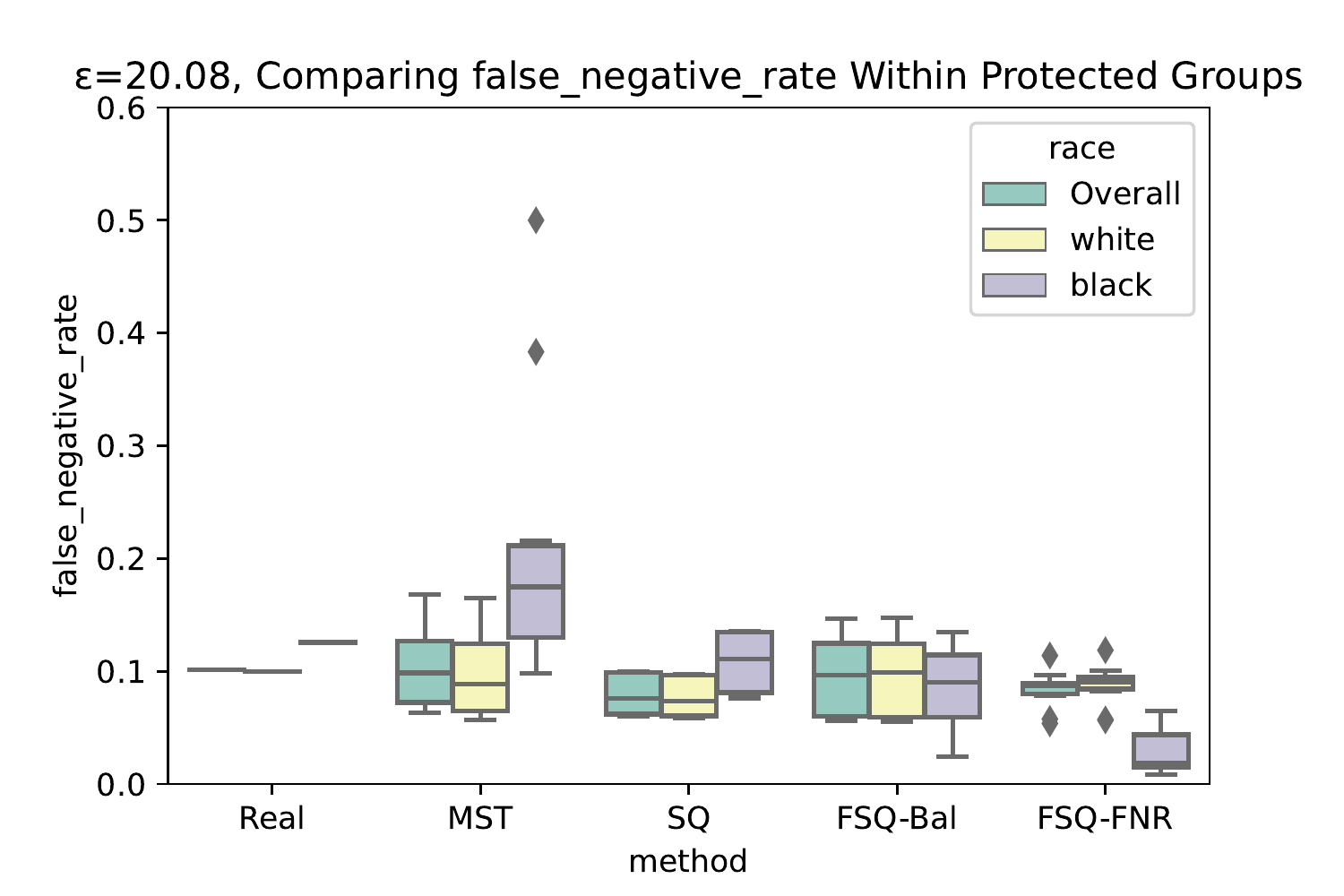}\hfill
\vspace{-0.1in}
\caption{False negative rate at varying $\epsilon \in \{e^0,e^2,e^3\}$ on \acse.}
\label{fig:fnr}
\vspace{-0.1in}
\end{subfigure}\par\medskip

\begin{subfigure}{\linewidth}
\centering
\includegraphics[width=.33\textwidth]{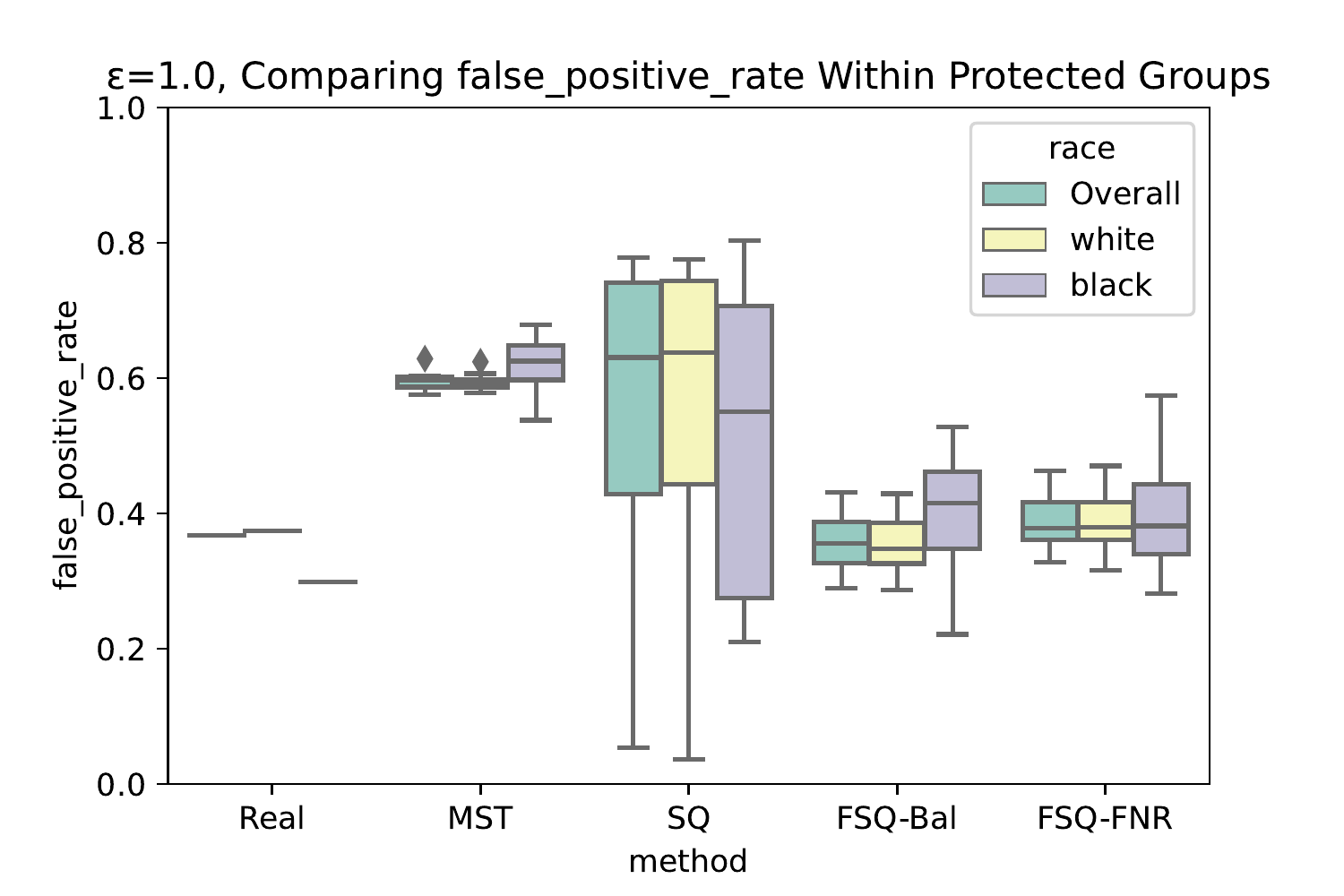}
\hfill
\includegraphics[width=.33\textwidth]{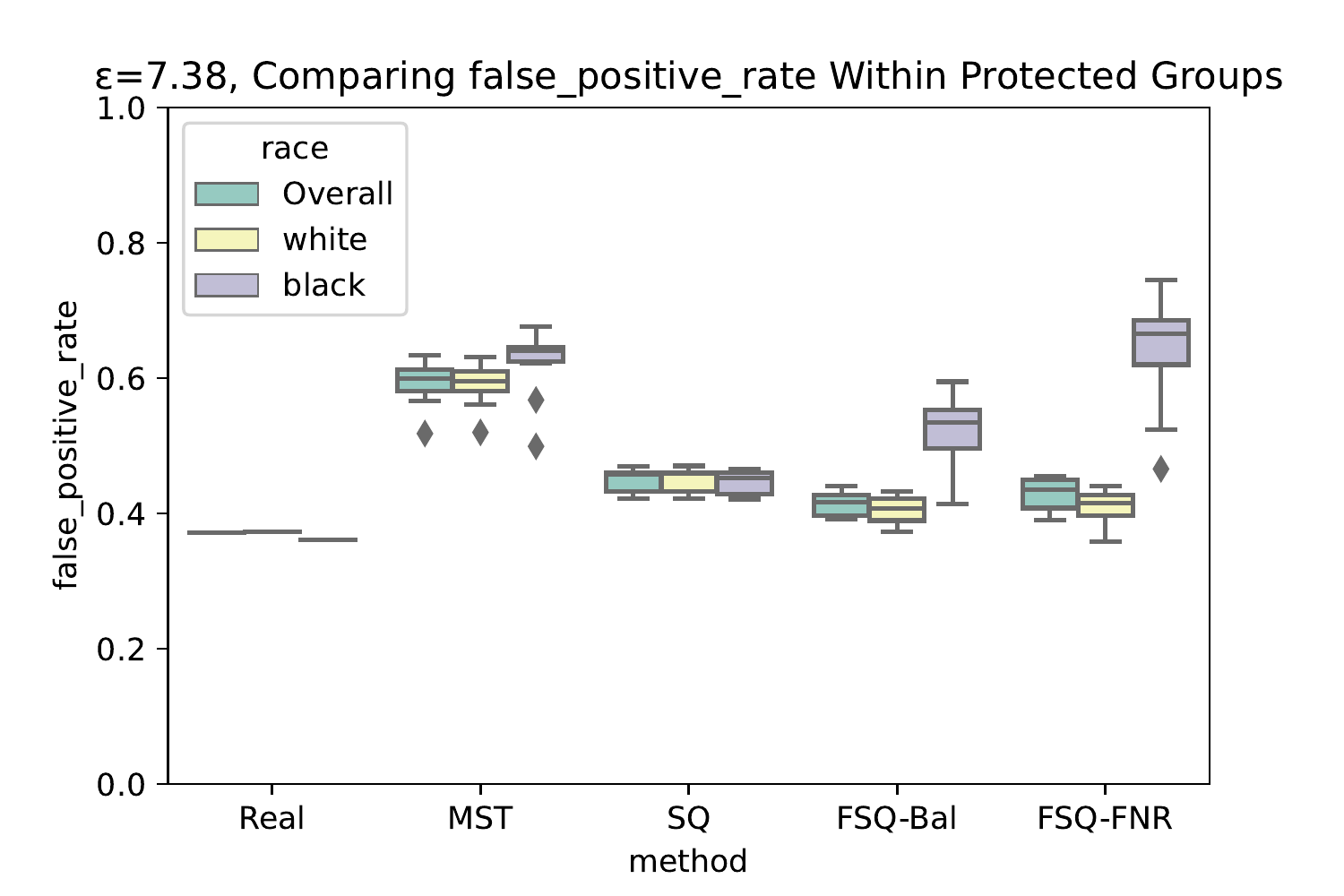}\hfill
\includegraphics[width=.33\textwidth]{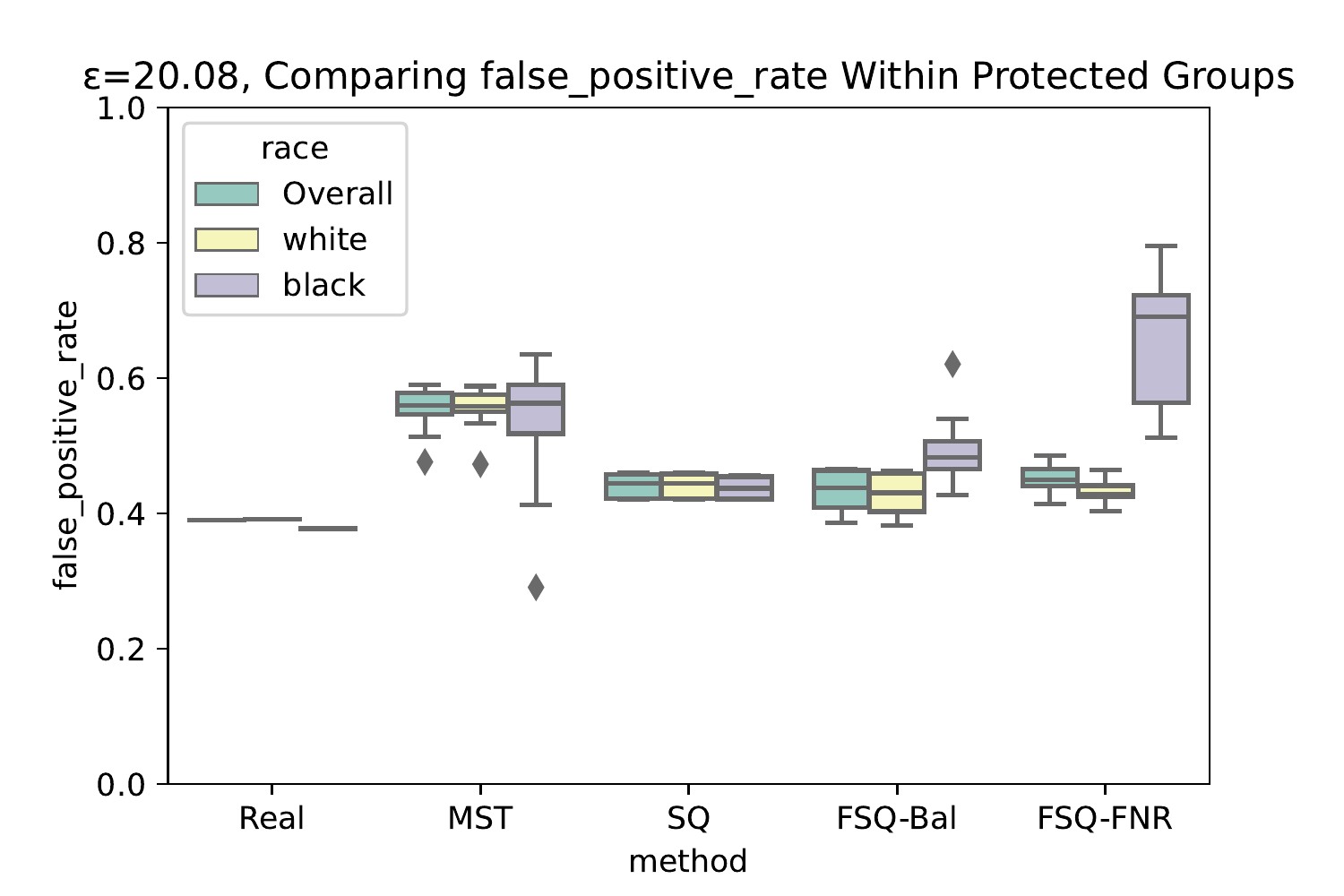}\hfill
\vspace{-0.1in}
\caption{False positives rate at varying $\epsilon \in \{e^0,e^2,e^3\}$ on \acse.}
\label{fig:fpr}
\vspace{-0.1in}
\end{subfigure}\par\medskip
\end{figure*}

\subsection{Privacy and Budget}
It is quite difficult to determine a ``reasonable'' $\epsilon$ budget for a DP mechanism, and in practice what constitutes a reasonable level of privacy varies by context \cite{dwork2019differential}. ``Small'' $\epsilon$ values (canonically, $\epsilon \leq 1$) tend to exhibit great privacy guarantees but often severely impact performance \cite{lee2011much}. ''Medium'' and ''large'' $\epsilon$ values (canonically, $\epsilon \in \{1,10\}$ and $\epsilon \geq 10$, respectively) provide more relaxed privacy guarantees, but increase utility. In this paper, we compare to the state-of-the-art data synthesis method MST \cite{mckenna2021winning}, which exhibits relatively high performance at ``small'' $\epsilon$ values, but is not explicitly geared to scenarios where an increased $\epsilon$ value would be acceptable (in our experiments,  performance may stagnate). We refer to the difference as ``excess privacy budget.'' 

Our methods, while not useful at $\epsilon$ values $\leq 1$ compared to MST, utilize this excess privacy budget at medium and large $\epsilon$ values to provide significant performance boosts in predictive analysis in some contexts. Furthermore, the small $\epsilon$ values have disparate impacts on minority populations in data \cite{ganev2021robin}. We work to address this issue, providing a method to balance or even eliminate harm done to the subgroup. When assessing our models, we borrow from the value range of \cite{bowen2020comparative}, and use values of $e^0, e^2,e^3,e^4$, or $\epsilon \approx 1.0, 7.38, 20.08, 54.59$, respectively.

Both DPSAGE and SuperQUAIL (and its derivatives) rely on the Standard Composition theorem to ensure differentially private results given an $\epsilon$ value \cite{dwork2014algorithmic}.
\begin{theorem}\textbf{Standard Composition Theorem}
Let $M_1: N^{|X|} \rightarrow R_1$ be an $\epsilon_1$-differentially private algorithm, and let $M_2: N^{|X|} \rightarrow R_1$ be $\epsilon_2$-differentially private algorithm. Then their combination, defined to be $M_{1,2}\rightarrow R_1XR_2$ by the mapping: $M_{1,2}(x)=(M_1(x),M_2(x))$ is $(\epsilon_1+\epsilon_2)$-differentially private.
\end{theorem}
\begin{theorem}
DPSAGE and SuperQUAIL are $\epsilon, \delta$-differentially private.
\end{theorem}
\begin{proof}
The simple proof of this theorem relies on an additive analysis based on the Composition Theorem, omitted due to space constraints (see Appendix for full details).
\end{proof}

\begin{algorithm}[tb]
\small 
\caption{SuperQUAIL}
\label{alg:sq}
\textbf{Input}: Real data $D$, $DPSAGE$ model \\
\textbf{Parameter(1)}: Budget $\epsilon$, split $\alpha$, feat importance threshold $\beta$ \\
\textbf{Parameter(2)}: Untrained DP Synthesizer $S$, DP supervised learning class $C$, model hyperparameters $H_S$ and $H_C$, target feature $f^*$\\
\textbf{Output}: DP synthetic data $\hat{D}$
\begin{algorithmic}[1] 
\STATE Split $\epsilon$ budget, $\epsilon_{S} = \alpha * \epsilon$ and $\epsilon_{C} = 1 - (\alpha * \epsilon)$ \label{lst:line:eps1}
\STATE Further subdivide $\epsilon_{C_i} = \frac{\epsilon_{C}}{\beta}$ \label{lst:line:eps2}
\STATE From $DPSAGE(\epsilon_{S})$, retrieve feature ranking $F$, values $V_F$, perpetuate trained $S(H_S,D)$, $C_{f^*}(H_C,D)$ for use later.
\FOR{$f$ in $F[:\beta]$}\label{lst:line:fit} 
\STATE \textbf{Fit step:} Train $C_i = C(\textit{target}=f, \epsilon_{C_i}, D, H_C)$. 
\ENDFOR
\FOR{Range of $|n|$, desired samples} 
\STATE \textbf{Sample step:} Generate sample $s \sim S$ \label{lst:line:sample}
\FOR{$C_i$ in random order of $C = \{C_1,C_2...C_{\beta}\}\cup \{C_{f^*}\}$} 
\STATE Replace $s[i] = x$ with $x = C_i(s'| s' = s \setminus \{f_i\})$. 
\ENDFOR
\STATE Add modified $s$ to synthetic dataset, $\hat{D} \cup \{s\}$
\ENDFOR
\STATE \textbf{return} DP synthetic data $\hat{D}$
\end{algorithmic}
\end{algorithm}

\vspace{-0.2in}
\section{Fair SuperQUAIL}
\label{sec:fair}
We make use of the inherent tunability of each of the components of SuperQUAIL to encourage a fairness target. For FSQ-Bal and FSQ-FNR, we make the following modifications to Algorithm~\ref{alg:sq}.

{\bf (1)} Within DPSAGE (line~\ref{lst:line:eps2}), declare a sensitive feature (e.g., race). Divide the DPSAGE $\epsilon$ budget between DP classifiers, each trained solely on the samples from one of the groups defined by the sensitive feature (e.g., race=black and race=white).\footnote{For ease of exposition, we discuss fairness w.r.t. a binary sensitive feature; multinary sensitive features are handled analogously.} Calculate feature importance for each group-wise classifier, since different features may explain classification of different groups; compute ranked lists per group.

{\bf (2)} Before training the classifiers (line~\ref{lst:line:fit}), round-robin through the per-group ranked lists of features, selecting 1 at a time from each list, until a total of $\beta$ features is selected.   Preferentially choose the protected group (e.g., race=black) on the last round.


{\bf (3)} Initialize FSQ-Bal or FSQ-FNR by specifying a weighted target constraint: maximize accuracy while equalizing false-negatives between groups for FSQ-Bal, or while minimizing false negatives for the protected group for FSQ-FNR. Within the iterative sampling step  (line~\ref{lst:line:sample}), tune the binary probability threshold $p$ for the target classifier: start at $p=0.5$, perturb until constraints are met. 

\section{Experimental Evaluation}
\label{sec:exp}

{\bf Data sets.} Many papers on DP synthesizers rely on standard machine learning data benchmarks, like UCI Adult and UCI Mushroom. These data scenarios tend to be unrealistic, and, in the case of UCI Adult, represent odd trends that are not reflected in comparable data today~\cite{ding2021retiring}. In light of this, we perform our tests on different US state specific Census scenarios, based on very recent data (from 2018), following guidance from~\cite{ding2021retiring} on better data scenarios for fairness. The predictive analysis is thus more difficult, but realistic. All experiments in this section were executed for three binary classification scenarios from the \href{https://github.com/zykls/folktables}{Folktables Project}: predicting employment status on \acse (state=CA, 17 features); predicting whether a low-income individual, not eligible for Medicare, has coverage from public health insurance on \acsh (state=NM, 19 features); and predicting whether a young adult moved addresses in the last year on \acst (state=NM, 21 features). We selected these scenarios based on findings in~\cite{ding2021retiring}, to explore a range of conditions for accuracy and fairness.
We decided to focus on US States studied at length in their paper, and so went with ACSEmployment for CA from 2018 (a larger dataset, which has n=378,817 samples), and NM for ACSPublicCoverage and ACSMobility, (a smaller dataset which has n=7,693 samples). Furthermore, both CA for ACSEmployment and NM for ACSPublicCoverage/ACSMobility had relatively balanced base rates for classification ($\approx 50\%$ positive/negative examples), which made preserving the predictive analysis through data synthesis more challenging.

{\bf Settings.} We compare performance of a DP classifier on real data (Real) to that on DP synthetic data generated by our methods (SQ, FSQ-Bal, FSQ-FNR) and by a strong baseline (MST)~\cite{mckenna2021winning}.  SuperQUAIL (SQ), presented in Section~\ref{sec:wise}, is an ensemble synthesizer that allocates the privacy budget in-line with feature importance. Fair SuperQUAIL (FSQ), presented in Section~\ref{sec:fair}, allocates privacy budget in-line with per-group feature importance.  Its variant FSQ-Bal balances accuracy across groups, while FSQ-FNR minimizes the FNR for the disadvantaged group.  Recall from Section~\ref{sec:wise} that SQ, FSQ-Bal, and FSQ-FNR all use MST as their embedded DP synthesizer.

\begin{table}[t!]
\centering
\small 
\caption{DPSAGE average performance (over 10 runs) at varying epsilons for top-5 ranked features and their SAGE values. }
\begin{tabular}{lrrrr}
\toprule
Scenario  & $\epsilon_1/\epsilon_2$ & nDCG & Jaccard & MAP \\
\midrule
\acse      & 0.2/1.0  & 0.73/0.97 & 0.54/0.80  & 0.58/0.81      \\
\acsh    & 0.2/1.0  & 0.40/0.59  & 0.30/0.41  &  0.33/0.44     \\
\acst  & 0.2/1.0  & 0.96/0.98   & 0.64/0.70  & 0.66/0.71    \\
\bottomrule
\end{tabular}
\vspace{-0.1in}
\label{tab:sageresults}
\end{table}

We experiment with a range of privacy budgets, $\epsilon \in \{ e^0, e^1, e^2, e^3\}$, with 10 runs per experiment.  All DP synthesizers receive the same overall privacy budget.  Due to space constraints, we present detailed results for  $\epsilon \in \{ e^0, e^2, e^3\}$ for \acse, using DPLogisticRegression~\cite{diffprivlib} 
for the classification task.  Results for other scenarios are summarized here and detailed in the Appendix. 

In Algorithm~\ref{alg:sq}, lines~\ref{lst:line:eps1}-~\ref{lst:line:eps2}, the $\gamma,\alpha$ epsilon splits for SQ, FSQ-Bal and FSQ-FNR in Figures~\ref{fig:accuracy}-\ref{fig:fpr} was decided via a grid search over $[0.2,0.8]$ at intervals of 0.1, and we also grid search over $\beta \in \{2,3,4,5\}$ feature importance thresholds. We observed experimentally that, at low epsilon values ($\epsilon=e^0,e^2$), an overall $\alpha$-split of $\{0.7\epsilon,0.3\epsilon\}$ between DPSAGE and the classifiers often performed best, while the internal DPSAGE  $\gamma$-split was most often best at $\{0.5\epsilon,0.5\epsilon\}$. For larger $\epsilon$ values ($\epsilon = e^3, e^4$), we found an even split ($\{0.5\epsilon,0.5\epsilon\}$) across all components performed well.

{\bf Feature importance with DPSAGE.} In this experiment, we evaluate the accuracy of DPSAGE for feature importance estimation by comparing the top-5 features selected by DPSAGE to those computed over the real data.  Table ~\ref{tab:sageresults} presents these results for nDCG, Jaccard similarity, and MAP. (Note that 1 is the best possible value for all these metrics.) 

Even at low epsilon values (such as $\epsilon=0.2$) the parity between real data and DP scores is quite good, although exactly \textit{how} good depends on context. Note that we are using lower $\epsilon$ values in Table~\ref{tab:sageresults} than in the end-to-end experiments below, since DPSAGE gets only a portion of the budget. Figure~\ref{fig:examplesage} gives a visual comparison of feature importance values for Real SAGE and DPSAGE, for a run on \acse.




\begin{figure}[t!]
\centering
\includegraphics[width=7cm]{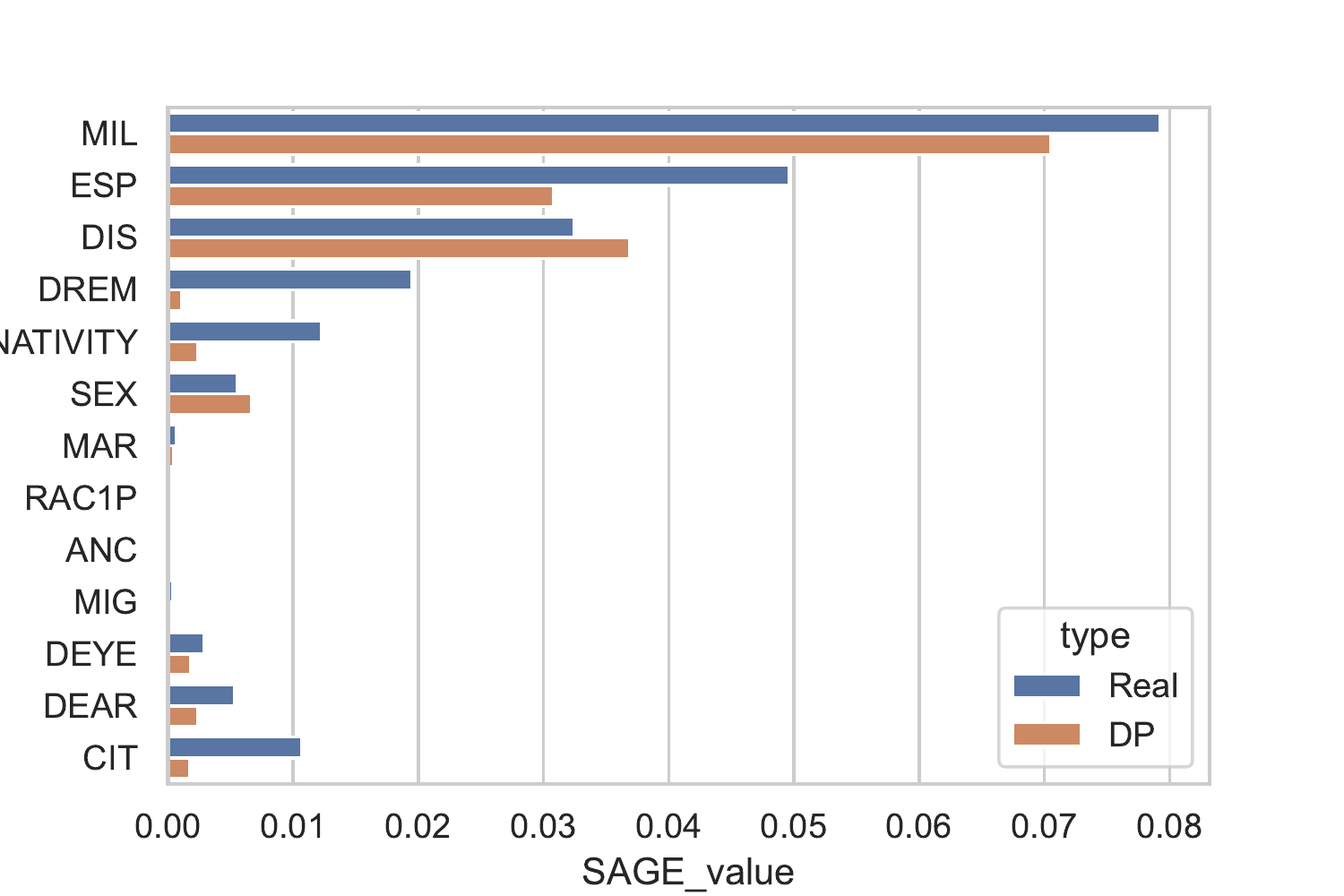}
\caption{A sample plot of feature importance values for Real SAGE and DPSAGE at $\epsilon=1.0$ for ACSEmployment}
\label{fig:examplesage}
\vspace{-0.1in}
\end{figure}

{\bf Performance of SuperQUAIL and Fair SuperQUAIL.} We now evaluate the end-to-end-performance of SuperQUAIL (SQ) and FairSuperQUAIL (FSQ-Bal and FSQ-FNR), comparing them to performance on real and MST-generated data. Figures~\ref{fig:accuracy} and~\ref{fig:f1} show DP classifier accuracy and f1 score, respectively. Observe that SQ and FSQ exhibit higher variance compared to MST at $\epsilon = e^0$, and thus an unreliable advantage over MST for a low privacy budget.  However, beginning with $\epsilon = e^1$ (not shown), and particularly for $\epsilon = \{ e^2, e^3 \}$, SQ outperforms MST, both overall and in each white and black group. This points to the effectiveness of the SQ framework, and particularly to incorporating feature importance into budget allocation. A similar trend holds for \acst , while \acsh requires a higher privacy budget $\epsilon = \{ e^3, e^4 \}$  (see Appendix). It is worth noting that MST outperforms SQ in parity with real data on measures of aggregate statistical fidelity (i.e., mean, variance, etc.), though this is not the target of our work.

For our fair interventions, note that FSQ-Bal balances accuracy across groups, while outperforming MST in both accuracy and f1, albeit with higher variance. 

Finally, to evaluate the effectiveness of different FSQ strategies on \acse, we present false negative rates (FNR) and false positive rates (FPR) in Figures~\ref{fig:fnr} and~\ref{fig:fpr}, respectively.  Based on these results, we observe that FSQ-Bal is, indeed, able to balance FPR and FNR across groups.  Further, we observe that FSQ-FNR achieves lower FNR for the black group even compared to the DP classifier on real data, affirming the effectiveness of our methods.

\balance 
\vspace{-0.1in}
\section{Conclusion and Future Work}
We presented SuperQUAIL, and its fair derivatives, to address shortcomings of DP data synthesis. SQ better allocates $\epsilon$ in settings where improving performance for prediction relies on capturing complex conditional dependence in synthetic data, all while providing mechanisms to reduce minority group harms exacerbated by data privatization.

\textbf{Limitations. }Marginal-based DP data synthesizers, such as MST, outperform SuperQUAIL at low $\epsilon$ values, though these high-privacy settings present dangerous fairness concerns. However, when fairness is less of a concern, or when predictive analysis is simpler, a general purpose (marginal based) method, such as MST, may be a better choice.

\textbf{Future Work. }SuperQUAIL is inherently limited by the quality of the embedded DP models; as these methods improve, the SuperQUAIL framework itself is expected to achieve better performance. 
With this work, we hope to encourage others to prioritize explicit mechanisms for addressing adverse harms brought about by data privatization in their DP synthesizers. Exploration into methods for fair data synthesis that work at low $\epsilon$ values, or theoretical analysis concluding that such systems are unfeasible, is sorely needed.

\newpage 






\small{
\bibliographystyle{named}
\bibliography{ijcai22}
}
\appendix

\LARGE\textbf{Appendix}\normalsize

\section{Results}
Below we summarize some findings in the other two data scenarios we explored from the \href{https://github.com/zykls/folktables}{$folktables$} project, which also rely on recent (2018) US Cencus Data. Both are data from the state of New Mexico. 
\subsection{ACSMobility}
We tested MST and SQ on the Mobility scenario, which predicts whether a person moved between residential addresses over the course of a year (only including individuals between the ages of 18 and 35, who are more likely to move).

Figures~\ref{fig:accuracymobility}-\ref{fig:fprmobility} contain results for \acst on the Real, MST and SuperQUAIL SQ synthesizers for $\epsilon \in \{ e^2, e^3, e^4\}$

Note that, at these medium/large $\epsilon$ values, the SuperQUAIL method matches the accuracy and F1 scores of the real data, whereas MST stagnates in its performance. However, the variance in performance on different subgroups of the population, and general difficulty of this scenario for predictive analysis (even the real data is skewed heavily) make it difficult to synthesize this data with any confidence in one-off fashion. This illustrates one of the limitations of DP synthetic data, when the data itself is noisy or inconsistent, the added noise can be preventatively difficult to deal with.

\begin{figure*}[t!]
\begin{subfigure}{\linewidth}
\centering
\includegraphics[width=.33\textwidth]{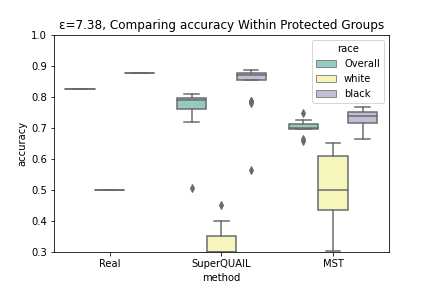}\hfill
\includegraphics[width=.33\textwidth]{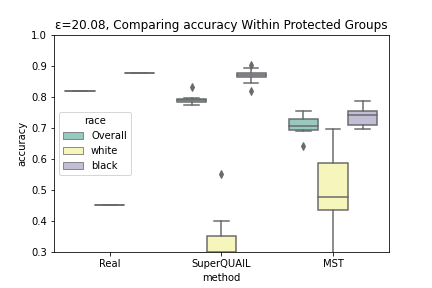}\hfill
\includegraphics[width=.33\textwidth]{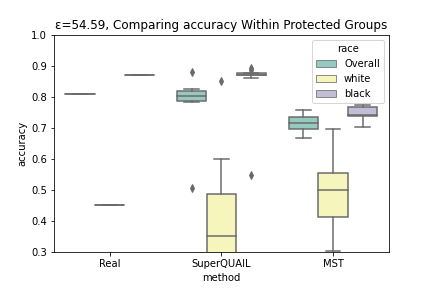}\hfill
\caption{Accuracy score at varying $\epsilon \in \{e^2,e^3,e^4\}$ on \acst.}
\label{fig:accuracymobility}
\end{subfigure}\par\medskip

\begin{subfigure}{\linewidth}
\centering
\includegraphics[width=.33\textwidth]{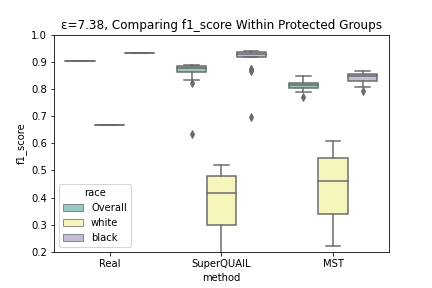}\hfill
\includegraphics[width=.33\textwidth]{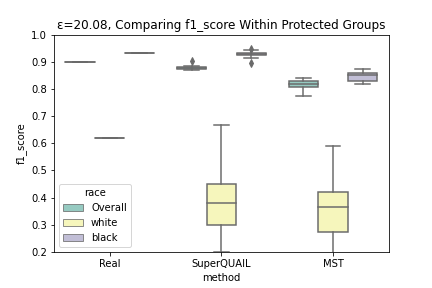}\hfill
\includegraphics[width=.33\textwidth]{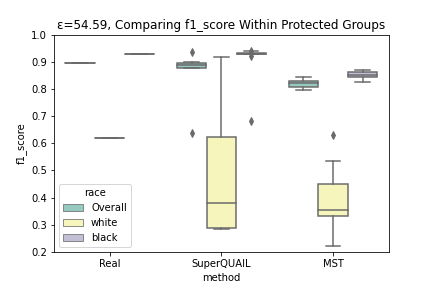}\hfill
\caption{F1 score at varying $\epsilon \in \{e^2,e^3,e^4\}$ on the \acst.}
\label{fig:f1mobility}
\end{subfigure}\par\medskip

\begin{subfigure}{\linewidth}
\centering
\includegraphics[width=.33\textwidth]{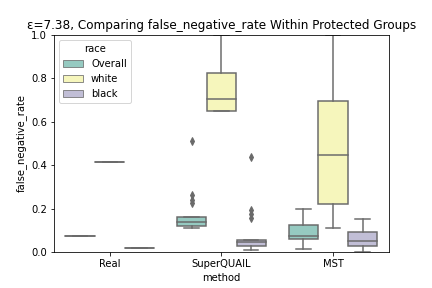}\hfill
\includegraphics[width=.33\textwidth]{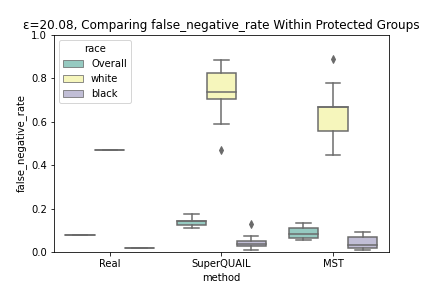}\hfill
\includegraphics[width=.33\textwidth]{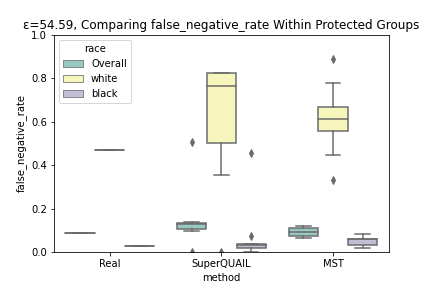}
\hfill
\caption{False negative rate at varying $\epsilon \in \{e^2,e^3,e^4\}$ on \acst.}
\label{fig:fnrmobility}
\end{subfigure}\par\medskip

\begin{subfigure}{\linewidth}
\centering

\includegraphics[width=.33\textwidth]{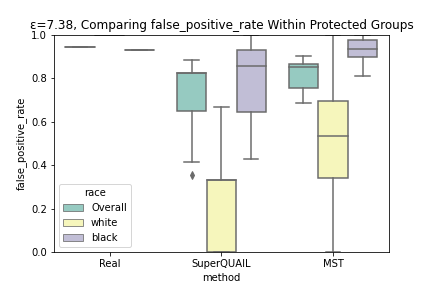}\hfill
\includegraphics[width=.33\textwidth]{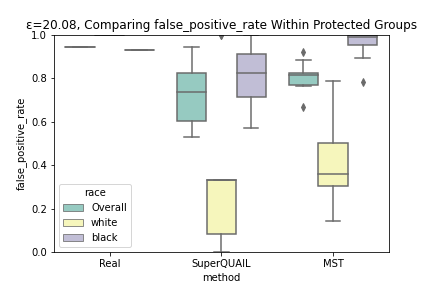}\hfill
\includegraphics[width=.33\textwidth]{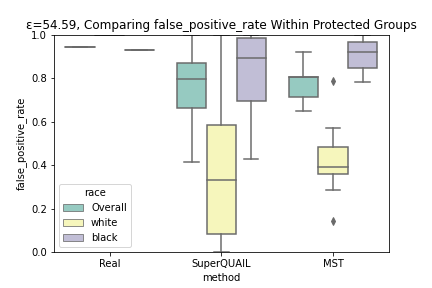}
\hfill
\caption{False positives rate at varying $\epsilon \in \{e^2,e^3,e^4\}$ on \acst.}
\label{fig:fprmobility}
\end{subfigure}\par\medskip
\end{figure*}

\subsection{ACSPublicCoverage}
We tested MST and SQ on the Public Coverage scenario, which predicts a persons by public health insurance status (has/doesn't have) - data filtered to include only individuals under the age of 65 making less than $\$30,000$. 

Figures~\ref{fig:accuracypublic}-\ref{fig:fprpublic} contain results for \acsh on the Real, MST and SuperQUAIL SQ synthesizers for $\epsilon \in \{ e^2, e^3, e^4\}$.

The \acsh  scenario exhibits similar behavior to the \acse scenario explored in the main body of the paper, although SuperQUAIL only begins to outperform MST at high privacy budgets $\epsilon \in \{ e^3, e^4\}$. The stark difference between Real data performance and the performance of the synthesizers during predictive analysis suggests that the target variables conditional dependence is noisy and sensitive to privatization.

\begin{figure*}[t!]

\begin{subfigure}{\linewidth}
\centering
\includegraphics[width=.33\textwidth]{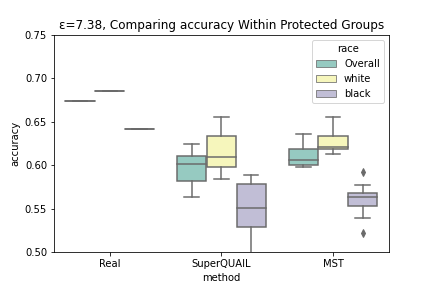}\hfill
\includegraphics[width=.33\textwidth]{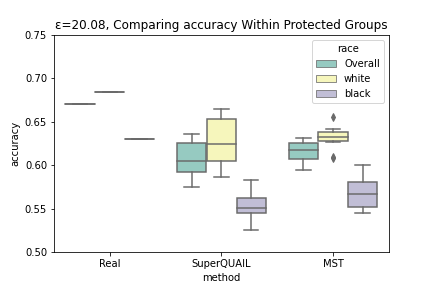}\hfill
\includegraphics[width=.33\textwidth]{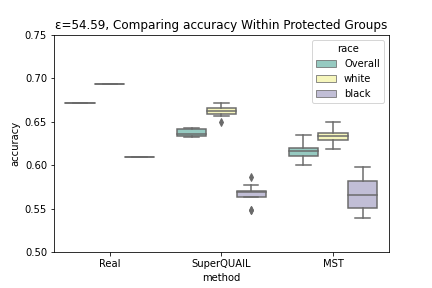}\hfill
\caption{Accuracy score at varying $\epsilon \in \{e^0,e^2,e^3\}$ on \acsh.}
\label{fig:accuracypublic}
\end{subfigure}\par\medskip

\begin{subfigure}{\linewidth}
\centering
\includegraphics[width=.33\textwidth]{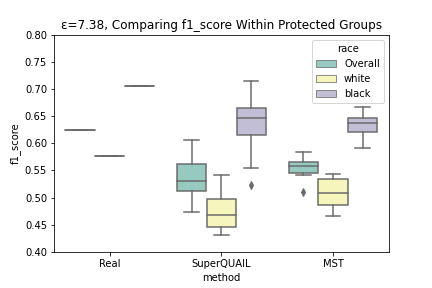}\hfill
\includegraphics[width=.33\textwidth]{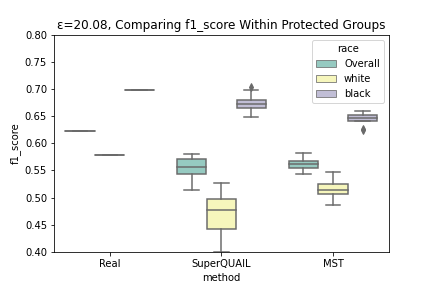}\hfill
\includegraphics[width=.33\textwidth]{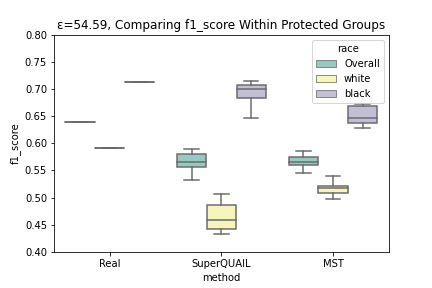}\hfill
\caption{F1 score at varying $\epsilon \in \{e^0,e^2,e^3\}$ on the \acsh.}
\label{fig:f1public}
\end{subfigure}\par\medskip

\begin{subfigure}{\linewidth}
\centering
\includegraphics[width=.33\textwidth]{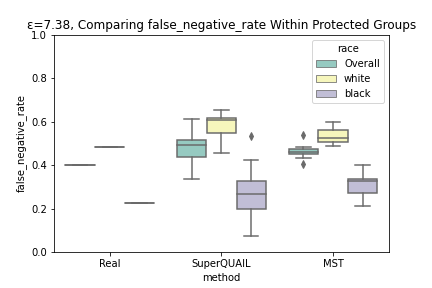}\hfill
\includegraphics[width=.33\textwidth]{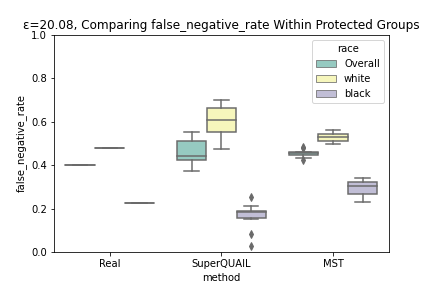}\hfill
\includegraphics[width=.33\textwidth]{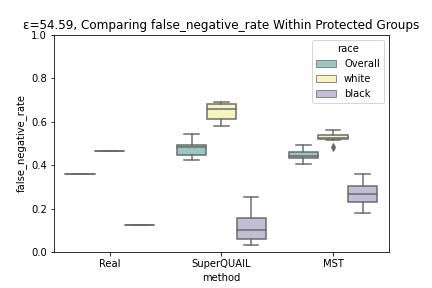}
\hfill
\caption{False negative rate at varying $\epsilon \in \{e^0,e^2,e^3\}$ on \acsh.}
\label{fig:fnrpublic}
\end{subfigure}\par\medskip

\begin{subfigure}{\linewidth}
\centering
\includegraphics[width=.33\textwidth]{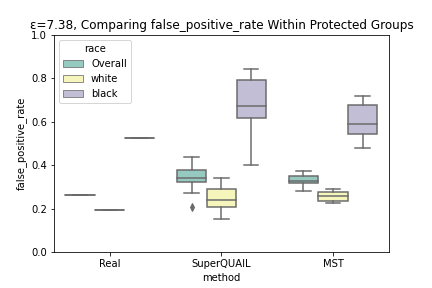}\hfill
\includegraphics[width=.33\textwidth]{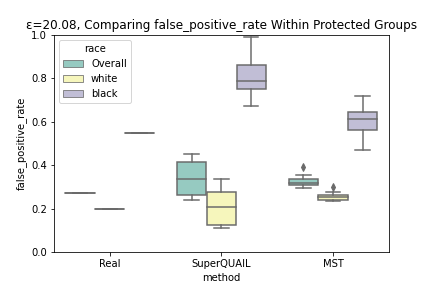}\hfill
\includegraphics[width=.33\textwidth]{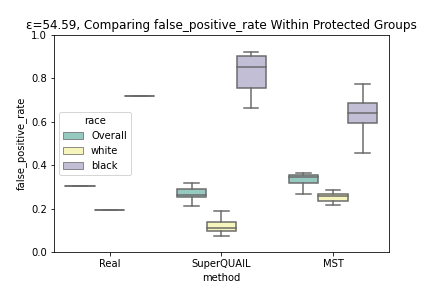}
\hfill
\caption{False positives rate at varying $\epsilon \in \{e^0,e^2,e^3\}$ on \acsh.}
\label{fig:fprpublic}
\end{subfigure}\par\medskip
\end{figure*}

\subsection{Replicability}
Note that the full evaluations, results, figure generating code, etc. is available at this \href{https://github.com/lurosenb/sq_supplementary_materials}{repository}.

\section{Proofs}
Proofs rely on notation derived from Algorithms \textbf{1} and \textbf{2} in main body of the paper.

\begin{theorem}
DPSAGE is $\epsilon, \delta$-differentially private.
\end{theorem}
\begin{proof}
\textit{DPSAGE}
Let the embedded DPSAGE Synthesizer, $S(H_S,\epsilon_S)$, be an $(\epsilon, \delta)$-differentially private mechanism mapping $M_1: N^{|X|} \rightarrow R_1$. Let the embedded DPSAGE Classifier, $C(H_C,\epsilon_C,target)$, be an $(\epsilon, \delta)$-differentially private mechanism mapping $M_2: N^{|X|} \rightarrow R_2$. Fix $0 < \gamma < 1$. Our original budget is $B$, thus by construction:
\begin{align}\label{dp}
    B &= \gamma * \epsilon + (1-\gamma) * \epsilon \\
    \epsilon_S &= \gamma * \epsilon \\
    \epsilon_C &= (1-\gamma) * \epsilon \\
    \text{By the \textit{Standard }}&\text{\textit{Composition Theorem} [Dwork et al.,2006]} \\
    \frac{Pr[M_{C,S}(x)=(r_1,r_2)]}{Pr[M_{C,S}(y)=(r_1,r_2)]}&\geq exp(-(\epsilon_S + \epsilon_C)) \\
    \frac{Pr[M_{C,S}(x)=(r_1,r_2)]}{Pr[M_{C,S}(y)=(r_1,r_2)]}&\geq exp(-(B))
\end{align}
\end{proof}

\begin{theorem}
SuperQUAIL is $\epsilon, \delta$-differentially private.
\end{theorem}
\begin{proof}
Fix $0 < \alpha < 1$. Set $\beta > 0$. By construction, we split original $\epsilon$ budget, $$\epsilon_{DPSAGE} = \alpha * \epsilon$$ and $$\epsilon_{C} = 1 - (\alpha * \epsilon)$$ we further subdivide $$\epsilon_{C_i} = \frac{\epsilon_{C}}{\beta}$$.

The $\epsilon_{C_i}$ budget is divided among $C = \{C_1,C_2...C_{\beta}\}$ The target feature classifier budget $\{C_{f^{*}}\}$ is consumed during the DPSAGE step, and is thus treated separately. Therefore, the budget for the embedded DP synthesizer $\epsilon_S$ and DP Classifiers $\epsilon_{C_i}$ sums to $\epsilon$:
$$
\epsilon = \beta \epsilon_{C_i} + (\epsilon_S + \epsilon_{C_{f^{*}}})
$$
and by the \textit{Standard Composition Theorem} maintains differential privacy.

\end{proof}

\section{Note: Datasets detail}
We have included tables describing the dataset features and their distributions in detail.

\begin{table*}
\npdecimalsign{.}
\nprounddigits{2}
\begin{tabular}{ln{1}{2}n{1}{2}n{1}{2}n{1}{2}n{1}{2}n{1}{2}n{1}{2}n{1}{2}n{1}{2}n{1}{2}n{1}{2}n{1}{2}n{1}{2}n{1}{2}}
\toprule
{} &            \text{MAR} &            \text{DIS} &            \text{ESP} &            \text{CIT} &            \text{MIG} &            \text{MIL} &            \text{ANC} &       \text{NAT} &           \text{DEAR} &           \text{DEYE} &           \text{DREM} &            \text{SEX} &          \text{RAC} &            \text{ESR} \\
\midrule
mean  &       3.054052 &       1.876402 &       0.591645 &       1.921218 &       1.235560 &       3.124163 &       1.703070 &       1.260793 &       1.964730 &       1.977712 &       1.851179 &       1.507211 &       3.100853 &       0.456165 \\
std   &       1.871732 &       0.329123 &       1.640849 &       1.527332 &       0.659972 &       1.588045 &       1.063244 &       0.439068 &       0.184463 &       0.147618 &       0.475799 &       0.499949 &       2.951456 &       0.498075 \\
min   &       1.000000 &       1.000000 &       0.000000 &       1.000000 &       0.000000 &       0.000000 &       1.000000 &       1.000000 &       1.000000 &       1.000000 &       0.000000 &       1.000000 &       1.000000 &       0.000000 \\
25\%   &       1.000000 &       2.000000 &       0.000000 &       1.000000 &       1.000000 &       4.000000 &       1.000000 &       1.000000 &       2.000000 &       2.000000 &       2.000000 &       1.000000 &       1.000000 &       0.000000 \\
50\%   &       3.000000 &       2.000000 &       0.000000 &       1.000000 &       1.000000 &       4.000000 &       1.000000 &       1.000000 &       2.000000 &       2.000000 &       2.000000 &       2.000000 &       1.000000 &       0.000000 \\
75\%   &       5.000000 &       2.000000 &       0.000000 &       4.000000 &       1.000000 &       4.000000 &       2.000000 &       2.000000 &       2.000000 &       2.000000 &       2.000000 &       2.000000 &       6.000000 &       1.000000 \\
max   &       5.000000 &       2.000000 &       8.000000 &       5.000000 &       3.000000 &       4.000000 &       4.000000 &       2.000000 &       2.000000 &       2.000000 &       2.000000 &       2.000000 &       9.000000 &       1.000000 \\
\bottomrule
\end{tabular}
\caption{\acse  scenario data description. Count=378817}
\end{table*}

\begin{table*}
\npdecimalsign{.}
\nprounddigits{2}
\begin{tabular}{ln{1}{2}n{1}{2}n{1}{2}n{1}{2}n{1}{2}n{1}{2}n{1}{2}n{1}{2}n{1}{2}n{1}{2}n{1}{2}n{1}{2}n{1}{2}n{1}{2}}
\toprule
{} &          \text{MAR} &          \text{SEX} &          \text{DIS} &          \text{CIT} &          \text{MIL} &          \text{ANC} &     \text{NAT} &         \text{DEAR} &         \text{DEYE} &         \text{DREM} &        \text{RAC} &          \text{COW} &          \text{ESR} &          \text{MIG} \\
\midrule
mean  &     0.417369 &     0.482719 &     0.923613 &     0.156550 &     0.164392 &     0.989834 &     0.079872 &     0.986930 &     0.015974 &     0.046181 &     1.886146 &     1.569852 &     0.939588 &     0.749927 \\
std   &     0.628410 &     0.499774 &     0.265655 &     0.565988 &     0.610520 &     0.630119 &     0.271135 &     0.113591 &     0.125395 &     0.209907 &     1.602345 &     2.116038 &     0.961102 &     0.433118 \\
min   &     0.000000 &     0.000000 &     0.000000 &     0.000000 &     0.000000 &     0.000000 &     0.000000 &     0.000000 &     0.000000 &     0.000000 &     0.000000 &     0.000000 &     0.000000 &     0.000000 \\
25\%   &     0.000000 &     0.000000 &     1.000000 &     0.000000 &     0.000000 &     1.000000 &     0.000000 &     1.000000 &     0.000000 &     0.000000 &     1.000000 &     0.000000 &     0.000000 &     0.500000 \\
50\%   &     0.000000 &     0.000000 &     1.000000 &     0.000000 &     0.000000 &     1.000000 &     0.000000 &     1.000000 &     0.000000 &     0.000000 &     1.000000 &     0.000000 &     1.000000 &     1.000000 \\
75\%   &     1.000000 &     1.000000 &     1.000000 &     0.000000 &     0.000000 &     1.000000 &     0.000000 &     1.000000 &     0.000000 &     0.000000 &     2.000000 &     3.000000 &     1.000000 &     1.000000 \\
max   &     4.000000 &     1.000000 &     1.000000 &     4.000000 &     3.000000 &     3.000000 &     1.000000 &     1.000000 &     1.000000 &     1.000000 &     7.000000 &     9.000000 &     5.000000 &     1.000000 \\
\bottomrule
\end{tabular}
\caption{\acst  scenario data description. Count=3443}
\end{table*}

\begin{table*}
\npdecimalsign{.}
\nprounddigits{2}
\begin{tabular}{ln{1}{2}n{1}{2}n{1}{2}n{1}{2}n{1}{2}n{1}{2}n{1}{2}n{1}{2}n{1}{2}n{1}{2}n{1}{2}n{1}{2}n{1}{2}n{1}{2}n{1}{2}n{1}{2}}
\toprule
{} &          \text{MAR} &          \text{SEX} &          \text{DIS} &          \text{ESP} &          \text{CIT} &          \text{MIG} &          \text{MIL} &          \text{ANC} &     \text{NAT} &         \text{DER} &         \text{DEY} &         \text{DRM} &          \text{ESR} &          \text{FER} &        \text{RAC} &       \text{PUB} \\
\midrule
mean  &     0.793709 &     0.542181 &     0.160666 &     0.216431 &     0.192903 &     0.851553 &     0.295333 &     0.452359 &     0.098271 &     0.032497 &     0.036267 &     0.072923 &     0.722215 &     0.388535 &     1.488236 &     0.466138 \\
std   &     1.033853 &     0.498250 &     0.367246 &     0.908062 &     0.633373 &     0.376175 &     0.876026 &     0.762296 &     0.297700 &     0.177328 &     0.186965 &     0.260028 &     0.922460 &     0.525688 &     2.122980 &     0.498884 \\
min   &     0.000000 &     0.000000 &     0.000000 &     0.000000 &     0.000000 &     0.000000 &     0.000000 &     0.000000 &     0.000000 &     0.000000 &     0.000000 &     0.000000 &     0.000000 &     0.000000 &     0.000000 &     0.000000 \\
25\%   &     0.000000 &     0.000000 &     0.000000 &     0.000000 &     0.000000 &     1.000000 &     0.000000 &     0.000000 &     0.000000 &     0.000000 &     0.000000 &     0.000000 &     0.000000 &     0.000000 &     0.000000 &     0.000000 \\
50\%   &     1.000000 &     1.000000 &     0.000000 &     0.000000 &     0.000000 &     1.000000 &     0.000000 &     0.000000 &     0.000000 &     0.000000 &     0.000000 &     0.000000 &     1.000000 &     0.000000 &     0.000000 &     0.000000 \\
75\%   &     1.000000 &     1.000000 &     0.000000 &     0.000000 &     0.000000 &     1.000000 &     0.000000 &     1.000000 &     0.000000 &     0.000000 &     0.000000 &     0.000000 &     1.000000 &     1.000000 &     3.000000 &     1.000000 \\
max   &     4.000000 &     1.000000 &     1.000000 &     8.000000 &     4.000000 &     2.000000 &     4.000000 &     3.000000 &     1.000000 &     1.000000 &     1.000000 &     1.000000 &     6.000000 &     2.000000 &     7.000000 &     1.000000 \\
\bottomrule
\end{tabular}
\caption{\acsh  scenario data description. Count=7693}
\end{table*}

\end{document}